\documentclass[format=acmsmall, review=false]{acmart}

\usepackage{acm-ec-19}

\usepackage{booktabs} 

\usepackage{amssymb}
\usepackage{graphicx}
\usepackage{url}
\usepackage{color}

\usepackage{color}
\usepackage{algorithm}
\usepackage{algorithmic}
\usepackage{hyperref}
\usepackage{bbm}
\usepackage{amsmath}
\usepackage{times}

\newtheorem{remark}{Remark}

\newcommand{\Ex}{{\mathbb{E}}}
\newcommand{\R}{{\mathbb{R}}}
\newcommand{\inv}{\text{inv}}

\newcommand{\mean}{\mu}

\newcommand{\trueCost}{\bar{\mathcal C}}
\newcommand{\Reg}{\text{Regret}}
\newcommand{\mrpx}{{\mathcal M}(x,\s_1)}
\newcommand{\mrpxs}{{\mathcal M}(x,\s)}
\newcommand{\s}{{\bf{s}}}
\newcommand{\range}{[0, \rmax]}

\newcommand{\todoS}[1]{}

\newcommand{\st}{{\mathcal S}}

\newcommand{\mypara}[1]{\vspace*{0.1in}\noindent {\bf #1}}
\newcommand{\event}{{\mathcal E}}
\newcommand{\rmax}{{U}}
\newcommand{\Hval}{{576\max(h,p)(L+1)\rmax}}

\newcommand{\Kval}{{\log_{4/3}(T)}}
\newcommand{\Cv}{{\bf{C}}}
\newcommand{\ci}{[LB(\Cv_N^a), UB(\Cv^a_N)]}
\begin{document}

\title[Improved regret bounds for inventory management]{Learning in structured MDPs with convex cost functions: Improved regret bounds for inventory management}

\author{Shipra Agrawal and Randy Jia}

\begin{abstract}
We consider a stochastic inventory control problem under censored demands, lost sales, and positive lead times. This is a fundamental problem in inventory management, with significant literature establishing near-optimality of a simple class of policies called ``base-stock policies'' for the underlying Markov Decision Process (MDP), as well as convexity of long run average-cost under those policies. We consider the relatively less studied problem of designing a learning algorithm for this problem when the underlying demand distribution is unknown. The goal is to bound regret of the  algorithm when compared to the best base-stock policy. We utilize the convexity properties and a newly derived bound on bias of base-stock policies to establish a connection to stochastic convex bandit optimization. 

Our main contribution is a learning algorithm with a regret bound of $\tilde{O}(L\sqrt{T}+D)$ for the inventory control problem. Here $L$ is the fixed and known lead time, and $D$ is an unknown parameter of the demand distribution described roughly as the number of time steps needed to generate enough demand for depleting one unit of inventory. Notably, even though the state space of the underlying MDP is continuous and $L$-dimensional, our regret bounds depend linearly on $L$. Our results significantly improve the previously best known regret bounds for this problem where the dependence on $L$ was exponential and many further assumptions on demand distribution were required. The techniques presented here may be of independent interest for other settings that involve large structured MDPs but with convex cost functions.
\end{abstract}

\maketitle

\section{Introduction}
\label{sec:intro}
Many operations management problems involve making  decisions sequentially over time, where the outcome of a decision may depend on the current state of the system in addition to an uncertain demand or customer arrival process. This includes several online decision making problems in revenue and supply chain management. There, the sales revenue and supply costs incurred as a result of pricing and ordering decisions may depend on the current level of inventory in stock, back orders, outstanding orders etc.,  in addition to the uncertain demand and/or supply for the products. Markov Decision Process (MDP) is a useful framework for modeling these sequential decision making problems. In a typical formulation, the state of the MDP captures the current position of inventory. The reward (observed sales) depends on the the current state of the inventory in addition to the demand. The stochastic state transition and reward generation models capture the uncertainty in demand. 

A fundamental yet notoriously difficult problem in this area is the periodic inventory control problem under {\it positive lead times} and {\it lost sales} \cite{zipkin2000foundations, zipkin2008old}. 
In this problem, in each of the $T$ sequential decision making periods, the decision maker takes into account the current on-hand inventory and the pipeline of outstanding orders to decide the new order. There is a fixed delay (i.e., lead time) between placing an order and receiving it. A random demand is generated from a static distribution independently in every period. However, the demand information is {\it censored} in the sense that  the decision maker observes only the sales, i.e., the minimum of demand and on-hand inventory. Any unmet demand is lost, and incurs a penalty called lost sales penalty. 
Any leftover inventory at the end of a period incurs a holding cost. The aim is to minimize the aggregate long term inventory holding cost and lost sales penalty. 
There is a significant existing research that develops a Markov model (or semi-Markov model due to unobserved lost sales penalty) for this problem, and studies methods for computing optimal policies, assuming the demand distribution is {\it either known or can be efficiently simulated} (e.g., see survey in \cite{bijvank2011lost}). 
In particular, a simple class of policies called base-stock policies have been shown to be asymptotically \footnote{with increase in lost sales penalty} optimal for this problem \cite{huh2009asymptotic, bijvank2011lost}. 
Under a base-stock policy, the inventory position is always maintained at a target ``base-stock level.'' Notably, when using a base-stock policy, the infinite horizon average cost function for the inventory control MDP can be shown to be convex in the base-stock level \cite{janakiraman2004lost}. Therefore, under known demand model, convex optimization can be used to compute the optimal base-stock policy. 

In this paper, we considered the relatively less studied problem of periodic inventory control when the decision maker {\it does not know the demand distribution a priori}. The goal is to design a learning algorithm that can use the observed outcomes of past decisions to implicitly learn the unknown underlying MDP model and adaptively improve the decision making strategy over time, aka a reinforcement learning algorithm.  Following the near-optimality of base-stock policies, we aim to bound regret of the learning algorithm when compared to the best base-stock policy as benchmark.

The two main challenges in designing an efficient learning algorithm for the inventory control problem described above are presented by the {\it censored demand} and the {\it positive lead time}.  The censored demand assumption results in an exploration-exploitation tradeoff for the learning algorithm. 
Since the decision maker can only observe the sales, which is the minimum of demand and the on-hand inventory for a product, 
the quality of samples available for demand estimation of a product depend crucially on the past ordering decisions. For example, suppose that due to past ordering policies, a certain product was maintained at a low inventory level for most of the past sales periods, then the higher quantiles of the demand distribution for that product would be unobserved. Therefore, in order to ensure accurate demand learning, large inventory states need to be sufficiently explored. However, this exploration needs to be balanced with the holding cost incurred for any leftover inventory. There has been recent work on exploration-exploitation algorithms for regret minimization in finite MDPs with regret bounds that depend linearly or sublinearly on the size of the state space and action space (e.g., \cite{jaksch2010near, bartlett2009regal, agrawal2017optimistic}). However, the positive lead time in delivery of an order results in a much enlarged state space (exponential in lead time) for the inventory control problem, since the state needs to track all the outstanding orders in the pipeline. There is a further issue of discretization, since the  state space (inventory position) and action space (orders) is continuous. Discretizing over a grid would give a further enlarged state space and action space. As a result, none of the above-mentioned reinforcement learning techniques can be applied directly to obtain useful regret bounds for the inventory control problem considered here.

The main insight in this paper is that even though the state space is large, the convexity of the average cost function under the benchmark policies (here, base-stock policies) can be used to design an efficient learning algorithm for this MDP. 
We use the relation between  {\it bias} and {\it infinite horizon average cost} of a policy given by Bellman equations, to provide a connection between stochastic convex bandit optimization and the problem of learning and optimization in such MDPs. Specifically, we build upon the algorithm for stochastic convex optimization with bandit feedback from \cite{agarwal2011stochastic} to derive a simple algorithm that achieves an $\tilde{O}(L\sqrt{T} + D)$ regret bound for the inventory control problem. Here, $L$ is the fixed and known lead time. And, $D$ is a parameter of the demand distribution $F$, defined as the expected number of independent draws needed from distribution $F$ for the sum to exceed $1$. Importantly, although our regret bound depends on $D$, our algorithm does not need to know this parameter.

Our regret bound substantially improves the existing results for this problem by \cite{zhang2017closing, huh2009adaptive}, where the regret bounds grow exponentially with the lead time $L$ (roughly as $D^L \sqrt{T}$), and many further assumptions on demand distribution are required for the bounds to hold. A more detailed comparison to related work is provided later in the text. 
More importantly, we believe that our algorithm design and analysis techniques can be applied in an almost blackbox manner for minimizing regret in other problem settings involving MDPs that have convex cost function under benchmark policies. 
Such convexity results are available for many other operations management problems, for example, for several  formulations of admission control and server allocation problems in queuing \cite{weber1980note, lee1983note, shanthikumar1987optimal}.
Therefore, the techniques presented here may be of independent interest. 

\mypara{Organization.} The rest of the paper is organized as follows.  In the next two subsections, we provide the formal problem definition and describe our main results, along with a precise comparison of our regret bounds to closely related work.
In section \S\ref{sec:technical}, we use an MRP (Markov Reward Process) formulation to prove some key technical results, including convexity and bounded bias  of base-stock policies. These insights form the basis of algorithm design and regret analysis in sections \S\ref{sec:algo} and \S\ref{sec:reg} respectively. We conclude in section \S\ref{sec:conc}.

\subsection{Problem formulation}
\label{sec:probdef}
We consider a single product stochastic inventory control problem with lost sales and positive lead times. The problem setting considered here is similar to the setting considered in \citep{zhang2017closing, huh2009asymptotic}.  In this, an inventory manager makes sequential decisions in discrete time steps $t=1,\ldots, T$. In the beginning of every time step $t$, the inventory manager observes the current inventory level $\inv_t$, and $L$ previous unfulfilled orders in the pipeline, denoted as $o_{t-L}, o_{t-L+1}, \ldots, o_{t-1}$, for a single product.   Here, $L \geq 1$ is the lead time defined as the delay (number of time steps) between placing an order and receiving it. Initially in step $1$, there is no inventory ($\inv_1=0$) and no unfulfilled orders. Based on this information, the manager decides the amount $o_t\in \R$ of the product to order in the current time step.


The next inventory position is then obtained through the following sequence of events. First, the order $o_{t-L}$ that was made $L$ time steps earlier, arrives, so that the on-hand inventory level becomes $I_t=\inv_t + o_{t-L}$. Then, an {\it unobserved} demand $d_t \ge 0$  is generated from an unknown demand distribution $F$, independent of the previous time steps. Sales is the minimum of the on-hand inventory and demand, i.e., sales $y_t:=\min\{I_t, d_t\}$.  The decision maker only observes the sales $y_t$ and not the actual demand $d_t$, the demand information is therefore {\it censored}. 
A holding cost of $h(I_t-d_t)^+$ is incurred on remaining inventory and a lost sales penalty of $p(d_t-I_t)^+$ is incurred on part of the demand that could not be served due to insufficient on-hand inventory. That is, the cost incurred at end of step $t$ is,
\begin{equation} 
\label{eq:trueCost}
\trueCost_t = h(I_t-d_t)^+ + p(d_t-I_t)^+.
\end{equation}
Here, $h$ and $p$ are pre-specified constants denoting per unit holding cost  and per unit lost sales penalty, respectively. Note that the lost sales and therefore the lost sales penalty is unobserved by the decision maker.

\begin{figure}[t]
\centering
\includegraphics[width = \textwidth]{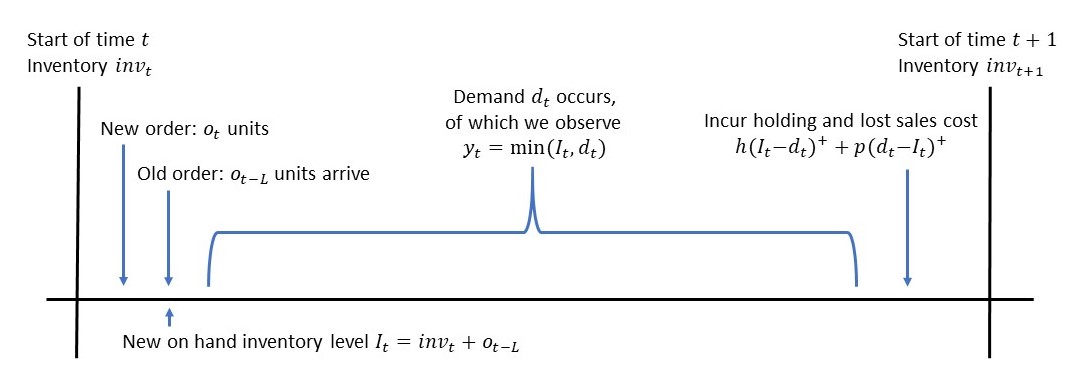}
\caption{Timing of arrival of orders and demand at time $t$.}
\label{fig:insert}
\end{figure}

Figure \ref{fig:insert} illustrates the timing of arrival of orders and demand. 
The next step $t+1$ begins with the leftover inventory 
\begin{equation}
  \inv_{t+1} := (I_t-d_t)^+ = (\inv_t + o_{t-L} -d_t)^+  
\end{equation}
and the new pipeline of outstanding orders $ o_{t-L+1}, \ldots, o_{t}$.

An online learning algorithm for this problem would sequentially decide the orders $o_1, \ldots, o_T$, under demand censoring, and without a priori knowing the demand distribution. The objective is to minimize the total expected cost $\Ex[\sum_{t=1}^T \trueCost_t$].

\mypara{Base-stock policies} aka order up to policies form an important class of policies for the inventory control problem.
Under this policy, the inventory manger always orders a  quantity  that brings the total inventory position (i.e., sum of leftover inventory plus outstanding orders) to some fixed value known as the base-stock level, if possible. Specifically,  let in the beginning of step $t$, the leftover inventory be $\inv_t$ and the outstanding orders be $o_{t-L}, \ldots, o_{t-1}$. Then, on using a base-stock policy with level $x$, the order $o_t$ in step $t$ is given by 
\begin{center}
$o_t=(x-\inv_t-\sum_{i=1}^L o_{t-i})^+$. 
\end{center}
\citep{huh2009asymptotic, zipkin2008old} provide empirical results that show that base-stock policies work well in many applications, and furthermore, \cite{huh2009asymptotic} show that as the ratio of lost sales to holding cost penalty increases to infinity, the ratio of the cost of the best base-stock policy to the optimal cost converges to $1$. Since the ratio of lost sales to holding cost penalty is typically large in applications, best base-stock policy can be considered close to optimal. 

\mypara{Regret against the best base-stock policy.}
Considering the asymptotic optimality of base-stock policies, several past works consider a more tractable objective of minimizing regret of an online algorithm compared to the best base-stock policy \citep{zhang2017closing, huh2009adaptive}. 

Let $\trueCost^x_t, t=1,2 \ldots, $ denote the sequence of costs incurred on running the base-stock policy with level $x$.  Define $\lambda^x$ as the expected infinite horizon  average cost of the base-stock policy, when starting from no inventory or outstanding orders, i.e.,
\begin{equation}
\label{eq:opt}
\lambda^x:= \Ex\left[\lim_{T\rightarrow \infty} \frac{1}{T} \sum_{t=1}^T \trueCost^x_t \left|\right.\  \inv_1=0
\right]
\end{equation}

Following result from \cite{janakiraman2004lost} shows that 
this long-run average cost is convex in $x$. 
\begin{lemma}[derived from Theorem 12 of \cite{janakiraman2004lost}] \label{lem:convex_ref}
 Given a demand distribution $F$ such that $F(0) > 0$, i.e., there is non-zero probability of zero demand. Then, for any $x \geq 0$, 
 the expected infinite horizon average cost (i.e., $\lambda^x$) is convex in $x$. 
\end{lemma}

\begin{remark}
 Theorem 12 of \cite{janakiraman2004lost} actually proves convexity of average cost when starting from an inventory level $\inv_1=x$. However, in definition of $\lambda^x$, we assumed starting inventory $\inv_1=0$. On starting from no-inventory and outstanding orders, and using base-stock policy with level $x$, the system will reach the state with inventory $x$ and no-outstanding orders in finite (exactly $L$) steps. Therefore, $\lambda^x$ is same as the expected infinite horizon average cost on starting with $\inv_1=x$. 
\end{remark}
Regret of an algorithm is defined as difference in its total cost in time $T$ compared to the asymptotic cost of the best base-stock policy. That is,

\begin{equation}
\label{eq:reg_def}
\Reg(T) := \Ex\left[\sum_{t=1}^T \trueCost_t \right] - T \left(\min_{x\in \range}\lambda^x\right)
\end{equation}
where $\trueCost_t, t=1,2\ldots,$ is the sequence of costs incurred on running the algorithm starting from no-inventory and no outstanding orders. $\range$ is some pre-specified range of base-stock levels to be considered.

\subsection{Main results}
Before we formally state our main theorem, we need to define $D$, 
a parameter of the demand distribution $F$ that appears in our regret bounds. However, it is important to note that our algorithm does not need to know the parameter $D$. 

\begin{definition} We define $D$ as the expected number of independent samples needed from distribution $F$  for the sum of those samples to exceed $1$. More precisely, let $d_1,d_2,  d_3,\ldots,$ be a sequence of independent samples generated from the demand distribution $F$, and let $\tau$ be the minimum number such that $\sum_{i=1}^\tau d_i \geq 1$. Then define $D := \Ex[\tau]$. We also refer to $D$ as the expected time to deplete one unit of inventory. We will assume that the demand distribution $F$ is such that $D$ is finite.
\end{definition}
Our main result is stated as follows.
\begin{theorem} 
\label{thm:reg}
Given a demand distribution $F$ such that $F(0) > 0$ and the expected time $D$ to deplete one unit of inventory is finite. Then, for $L\geq 0$, there exists an algorithm (Algorithm \ref{alg}) for the inventory control problem with regret bounded as:
\[\Reg(T) \leq  \tilde{O}\left( D\max(h,p)\rmax^2 + (L+1)\max(h,p)\rmax\sqrt{T}  \right),\]
 with probability at least $1-\frac{1}{T}$. 
For $T\geq (D\rmax)^2$, this implies a regret bound of 
\[\Reg(T) \leq \tilde{O}\left( (L+1)\max(h,p) \rmax\sqrt{T} \right).\]
Here $\tilde{O}(\cdot)$ hides logarithmic factors in $h, p, \rmax, L, T$, and absolute constants.
\end{theorem}
Here, constants $\max(h,p)$ and $\rmax$ define the scale of the problem. Note that the regret bound has a very mild (additive) dependence on the parameter $D$ of the demand distribution, defined as the expected time to deplete one unit of inventory. We conjecture that such a dependence on $D$ in the regret may be unavoidable; any time a learning algorithm reaches an inventory level higher than the optimal base-stock policy, 
it must necessarily wait time steps roughly proportional to $D$ for the inventory to deplete, in order to play a better policy. Only an algorithm that never overshoots the optimal inventory level may avoid incurring this waiting time. However, without a priori knowledge of the optimal level, an exploration based learning algorithm is unlikely to avoid this completely. This dependence also reminds of the dependence on diameter of an MDP in regret bounds for RL algorithms (e.g. see \cite{jaksch2010near, agrawal2017optimistic, tewari2008optimistic}). 
\begin{remark}
The condition $F(0)>0$ in the above theorem is required only for  using the result on convexity of infinite horizon average cost in Theorem 12 of \cite{janakiraman2004lost} (see Lemma \ref{lem:convex_ref}). The convexity result can in fact also be shown to hold under some alternate conditions like finite support of demand, or under sufficient discretization of demand.  
\end{remark}

\begin{remark}
\label{rem:alter}
One may consider an alternative regret definition that compares difference in total cost of the algorithm in time $T$ to the total cost of the best base-stock policy. That is,
\[\Reg'(T) := \Ex\left[\sum_{t=1}^T (\trueCost_t - \bar{C}_t^{x^*})\right]  \text{ where } x^*=\arg \min_{x\in [0,U]} \lambda^x. \]
  We show that our proof implies a bound similar to Theorem \ref{thm:reg} for this alternative regret definition. 
\end{remark}

\subsection{Comparison to related work}
Some earlier works on exploration-exploitation algorithms for inventory control problem \cite{huh2009nonparametric, besbes2013implications} provide $\tilde{O}(\sqrt{T})$ regret bounds, but under zero-lead time \cite{huh2009nonparametric} and/or perishable inventory \cite{besbes2013implications} assumptions.  
The inventory control problem considered here is exactly the same as that considered in recent work by \cite{zhang2017closing}, and earlier work by \cite{huh2009adaptive}. Therefore, we provide a precise comparison to the results obtained in those works. Our result matches the $O(\sqrt{T})$ dependence on $T$ in \cite{zhang2017closing}, improving on the $O(T^{\frac{2}{3}})$ dependence originally given in \cite{huh2009adaptive}. Further, it can be shown (see \cite{zhang2017closing}, Proposition 1) that for $T>5$, the expected regret for any learning algorithm in this setting is lower bounded by $\Omega(\sqrt{T})$, and thus our bound is optimal in $T$ (within logarithmic factors). 

More importantly, our regret bound scales linearly in $L$ as opposed to the exponential dependence on $L$ in \cite{zhang2017closing}.
Specifically, the regret bound achieved by \cite{zhang2017closing} is of order $\tilde{O}(\max(h,p)^2\rmax^2\left(\frac{1}{c}\right)^L \sqrt{T})$. Besides exponential dependence on $L$, the constant $c$ here is given by product of some positive probabilities for demand to take values in certain ranges, which requires several further assumptions on distribution $F$ (see Assumption 1 of \cite{zhang2017closing}). 

Among other related work, the results in \cite{lugosi2017hardness, bartok2014partial, besbes2015non} imply $\tilde{O}({\sqrt{T}})$ bounds for variations of inventory control problems under {\it adversarial demand}, however, under significant simplifying assumptions such as all remaining inventory perishes at the end of time period and there is no lead time. Under such assumptions there is no state dependence across periods and the problem becomes closer to online learning.

\section{MRP formulation and some key technical results}
\label{sec:technical}
Our algorithm design and analysis will utilize some key structural properties provided by base-stock policies. 
Specifically, we prove properties of a Markov Reward Process (MRP) obtained on running a base-stock policy for the inventory control problem. An MRP extends a Markov chain by adding a reward (or cost) to each state. In particular, the stochastic process obtained on fixing a policy in an MDP is an MRP. 

To define the MRP studied here, we observe that if we start with an  on-hand inventory and a pipeline of outstanding orders that sum to  less than or equal to $x$, then using base-stock policy $x$ will order the amount $o_t$ to bring the sum  to exactly $x$, i.e, 
\begin{center}
    $o_t=x-(\inv_t + \sum_{i=1}^{L} o_{t-i}) = x-I_t -\sum_{i=1}^{L-1} o_{t-i}$.
\end{center} From here on, the base-stock policy will always order whatever is consumed due to demand, i.e., $o_{t+1}=y_{t}$ where $y_{t} = \min\{I_t, d_t\}$ is the observed sales. And, the sum of inventory and outstanding orders will be maintained as $x$.

Based on this observation, we define an MRP with state at time $t$ as the tuple of available inventory and outstanding orders, i.e., $\s_t=(I_t, o_{t-L+1}, \ldots, o_t)$. The MRP  starts from a state where all the entries in this tuple sum to $x$. The base-stock policy will maintain this feature with the new state at time $t+1$ being  $\s_{t+1}=(I_t-y_{t}+o_{t-L+1}, o_{t-L+2}, \ldots, o_t, o_{t+1})$, where $o_{t+1}=y_t$ when using base-stock policy. 

We also define a cost $C^x(\s_t)$ associated with each state in this MRP. We define this as $C^x(\s_t)=\Ex[C^x_t |\s_t]$ where $C^x_t$ is a pseudo-cost defined as a modification of true cost $\trueCost^x_t$: 
\begin{equation}
\label{eq:cost1}
C^x_t=\trueCost^x_t-pd_t = h(I_t-y_t) - p y_t.
\end{equation}
The advantage of using this pseudo-cost is that since $I_t$ and $y_t$ are observable, the pseudo-cost is  completely observed. On the other hand, recall that the ``lost sales" in the true cost are not observed. Further, since the modification $pd_t$ does not depend on the policy or algorithm being used, later (see Lemma \ref{lem:reg_equiv}) we will be able to show that the regret computed using this version of the cost is in fact exactly the same as the regret $\Reg(T)$ in \eqref{eq:reg_def}.

Below is the precise definition of state space, starting state, reward model, and transition model of the MRP considered here. 


\begin{definition}[Markov reward process $\mrpx$]
We define MRP $\mrpx$ as the bipartite stochastic process 
\begin{center}$\{(\s_t, C^x(\s_t)); t=1,2,3, \ldots\}.$\end{center}
Here $\s_t$ and $C^x(\s_t)$ denotes state and the reward (cost) at time $t$ in  this MRP.  

Let $\st^x$ denote the set of $(L+1)$-dimensional non-negative vectors whose components sum to $x$.  The process starts in state $\s_1\in \st^x$. Given  state $\s_t = (s_t(0), s_t(1), \ldots, s_t(L))$,
 new state at time $t+1$ is given by
\begin{equation}
    \s_{t+1}:=(s_t(0)-y_t + s_t(1), s_t(2), \ldots, s_t(L), y_t)
\end{equation}
where $y_t=\min\{s_t(0), d_t\}, d_t\sim F$.
Observe that if $\s_1\in \st^x$, we have $\s_t\in \st^x$ for all $t$ by the above transition process. 
Cost function $C^x(\s_t)$ is defined as:
$$C^x(\s_t)= \Ex[C^x_t|\s_t]$$
where 
\begin{equation}
\label{eq:pseudoCost}
    C^x_t:=h (s_t(0)-y_t) - p y_t.
\end{equation}

\end{definition}



Two important quantities are the {\it loss} and {\it bias } of this MRP.
\begin{definition}[Loss and Bias]
\label{def:loss}
For any $\s \in {\st}^x$, loss $g^x(\s)$ and bias $v^x(\s)$ of MRP $\mrpxs$ are defined as:
$$\textstyle g^x(\s) :=\Ex\left[\lim_{T\rightarrow \infty }\frac{1}{T} \sum_{t=1}^T C^x(\s_t) | \s_1=\s \right]$$
$$\textstyle v^x(\s) :=\Ex\left[\lim_{T\rightarrow \infty }\sum_{t=1}^T C^x(\s_t) - g^x(\s_t) | \s_1=\s \right]$$
\end{definition}

\begin{remark}\label{rem:finite} 
Technically, for the above limits to exist, and for some other known results on MRPs  used later, we need finite state space and finite action space (see Chapter 8.2 in \cite{puterman2014markov}). Since we restrict to orders within range $\range$, and all states $\s\in \st^x$ are vectors in $[0,x]^{L}$ with $x\in \range$, 
we can obtain finite state space and action space by discretizing demand and orders using a uniform grid with spacing $\epsilon\in (0,1)$. Discretizing this way will give us a state space and action space of size $\left(\frac{\rmax}{\epsilon}\right)^{L}$ and $\frac{\rmax}{\epsilon}$, respectively. In fact, we can use arbitrary small precision parameter $\epsilon$, since our bounds will not depend on the size of state space or action space.  We therefore ignore this technicality in rest of the paper.
\end{remark}

The following lemma formally connects the loss of the MRP to the asymptotic average cost $\lambda^x$ of a base-stock policy (refer to \eqref{eq:opt}) used in defining regret. This connection will allow us to use the pseudo-costs instead of unobserved true costs. 
\begin{lemma}
\label{lem:connection}
Let $\s'\in \st^x$ be given by $\s':= (x,0,\ldots,0)$. Then, $$\lambda^x=g^x(\s')+p\mu$$ 
with  $\mean$ being the mean of the demand distribution $F$.
\end{lemma}
\begin{proof}
On using base stock policy with level $x$ starting in no inventory $x$ and no outstanding orders, the first order will be $x$, which will arrive at time step $L+1$. The orders and on-hand inventory will be $0$ for the first $L$ time steps $I_1=I_2=\ldots, I_L=0$. All the sales is lost, and therefore, the true cost in each of these steps is $pd_t$. In step $L+1$, we will have an on-hand inventory $I_{L+1}=x$ and no outstanding  orders. And, from here on, the system will follow a Markov reward process $\mrpxs$ with $\s_1=\s'$. 
Therefore, by relation  (see \eqref{eq:cost1}) between  pseudo-cost $C^x_t$ and true cost (lost sales penalty and holding cost) $\trueCost^x_t$, we have $C^x_t = \trueCost^x_t - p d_t$, for $t\ge L+1$. Therefore,

\begin{eqnarray*}
g^x(\s') & = & \lim_{T\rightarrow \infty } \Ex\left[\frac{1}{T} \sum_{t=L+1}^{L+T} C^x(\s_t) | \s_{L+1}=\s' \right]\\
& = & \lim_{T\rightarrow \infty } \Ex\left[\frac{1}{T} \sum_{t=1}^{T+L} (\trueCost_t^x - pd_t) \left|\right. \inv_1=0\right] \\
& = & \lambda^x - p\mu.
\end{eqnarray*}
\end{proof}
Next, we prove some important properties of bias and loss of this MRP, namely, 
\begin{enumerate}
\item that loss is  independent of the starting state and convex in $x$ (\S\ref{sec:prop1}), 
\item a bound on bias starting from any state (\S\ref{sec:prop2}), and  
\item a concentration lemma bounding the difference between the loss (i.e., expected infinite horizon average cost) and finite horizon average cost observed on running a base-stock policy (\S\ref{sec:prop3}).
\end{enumerate}
These results are presented in the next three subsections and will be crucial in algorithm design and analysis presented in the subsequent sections. 
To derive these properties, we first prove a bound on the difference in aggregate cost (termed as ``value") on starting from two different states. This result (proved in Lemma \ref{lem:value}) forms a key technical result utilized in proving all the above properties.
\begin{definition}[Value]
\label{def:value}
For any $\s \in {\st}^x$ the value $V^x_T(\s)$ in time $T$ of MRP $\mrpxs$ is defined as: $$\textstyle V^x_T(\s):= \Ex\left[\sum_{t=1}^T C^x(\s_t)|\s_1 =\s\right].$$
\end{definition}


\begin{lemma}[Bounded difference in value]\label{lem:value}
For any $x$, $T$, and $\s, \s'\in \st^x$,
\[V^{x}_T(\s) - V^{x}_T(\s') \leq 36\max(h,p)Lx.\]
\end{lemma}
\begin{proof}
For $L=0$, $\s = \s' = (x)$ and hence both sides are zero. Consider when $L \geq 1$. One way to bound the difference in the two values $V^x_T(\s)$ and $V^x_T(\s')$ is to upper bound the expected number of steps to reach a common state starting from $\s$ and $\s'$. Once a common state is reached, from that point onward, the two processes will have the same value. For example, if there is $0$ demand for $L$ consecutive time steps, then both processes will reach state $(x,0,\ldots, 0)$. Therefore, the difference in values can be upper bounded by a quantity proportional to inverse of the probability that demand is $0$ for $L$ consecutive steps. Unfortunately, this probability is exponentially small in $L$. In fact, the exponential dependence of regret on previous works (e.g., \cite{zhang2017closing, huh2009adaptive}) can be traced to using an argument like above somewhere in the analysis.
Instead we achieve a bound with linear dependence in $L$ by using a more careful analysis of the costs incurred on starting from different states. 

For any $\s \in \st^x$, we define $m^x_T(\s):=\sum_{t=1}^T I_t$ to be the total on-hand inventory level, and $n^x_T(\s):=\sum_{t=1}^T y_t$ to be the total sales in $T$ time steps, on starting from state $\s$. 
Then,  \[V^x_T(\s) := \Ex[\sum_{i=1}^T C^x_t |\s_1=\s]= \Ex[\sum_{i=1}^T h I_t - (h+p)y_t |\s_1=\s] = \Ex[h(m^x_T(\s)) - (h+p)(n^x_T(\s))].\] 

Thus, the difference between values $V^x_T(\s)$ and $V^x_T(\s')$ can be bounded by bounding difference in total on-hand inventory $|m^x_T(\s)- m^x_T(\s')|$ and total sales $|n^x_T(\s)- n^x_T(\s')|$. We bound this difference by first comparing pairs of states $\s,\s'$ that satisfy $\s'\succeq \s$, with the relation $\succeq$ defined as the property that for some index $k\ge 0$, the first $k$ entries satisfy $s'(0)\ge s(0), \ldots, s'(k)\ge s(k)$, and the remaining $L+1-k$ entries satisfy $s'(k+1)\le s(k+1), \ldots, s'(L)\le s(L)$.

For such pairs, we can bound the difference in total sales and total on-hand inventory as 
\[|n^x_T(\s)- n^x_T(\s')| \leq 3x \text{ and } |m^x_T(\s)- m^x_T(\s')| \leq 6Lx.\]
To see the intuition behind proving these bounds consider the sales observed on starting from $\s'$ vs. $\s$. We show that initially more sales are observed on starting from $\s'$, since $\s' \succeq \s$. Over time, the system keeps alternating between states with $\s'_t \succeq \s_t $ and $\s'_t \preceq \s_t $ in cycles of length at most $L$. The additional sales in one cycle with $\s'_t \succeq \s_t $ compensates for the lower sales in the next cycle with $\s'_t \preceq \s_t$, so that the total difference is bounded. The formal proofs for bounding difference in sales and on-hand inventory are provided in Lemma \ref{lem:tot_demand} and Lemma \ref{lem:tot_inv}, respectively, in the appendix.

Then we use the observation that $\hat{\s} \succeq \s$ for all  states $\s\in \st^x$ when $\hat{\s}:=(x,0,0,\dots,0)$. Therefore, we can apply the above results to conclude 
\[|n^x_T(\s)- n^x_T(\hat{\s})| \leq 3x \text{ and } |m^x_T(\s)- m^x_T(\hat{\s})| \leq 6Lx,\]
implying 
 \[|V^x_T(\s) - V^x_T(\hat{\s})| = \left|\Ex[h(m^x_T(\s) - m^x_T(\hat{\s})) - (h+p)(n^x_T(\s)-n^x_T(\hat{\s}))] \right| \leq 9(h+p)Lx.\]
Since the above holds for any state $\s$, we have that for two arbitrary starting states $\s,\s' \in \st^x$, \[|V^x_T(\s) - V^x_T(\s')| = |V^x_T(\s) -V^x_T(\hat{\s}) + V^x_T(\hat{\s}) - V^x_T(\s')| \leq 18(h+p)Lx \leq 36\max(h,p)Lx.\]

\end{proof}
\subsection{Uniform and convex loss}
\label{sec:prop1}
Next, we use the value difference lemma to show that the loss $g^x(\s)$ is independent of the starting state $\s\in \st^x$ in this MRP.

\begin{lemma}[Uniform loss lemma]\label{lem:uniformLoss}
For any $x$, $\s, \s'\in \st^x$, 
\[g^x(\s')=g^x(\s) =: g^x.\]
\end{lemma}
\begin{proof}
Using definition of $V^x_T(\s)$ and $g^x(\s)$,  $g^x(\s) = \lim_{T\rightarrow \infty}\frac{1}{T} V^x_T(\s)$ so that using Lemma \ref{lem:value}
\[|g^x(\s) -g^x(\s')| = |\lim_{T\rightarrow \infty}\frac{1}{T}V^x_T(\s) - \lim_{T\rightarrow \infty}\frac{1}{T}V^x_T(\s')| \leq \lim_{T\rightarrow \infty} \frac{36\max(h,p)Lx}{T} = 0, \]
since both limits exist (see Remark \ref{rem:finite}). Hence for any $\s, \s' \in \st^x$, $g^x(\s')=g^x(\s)$.
\end{proof}

Now the convexity of $g^x$ follows almost immediately from convexity of $\lambda^x$ and the relation given in Lemma \ref{lem:connection}

\begin{lemma}[Convexity lemma]
\label{prop:convexity}
Assuming demand distribution $F$ is such that there is a constant probability of $0$ demand, i.e., $F(0)>0$. Then, for any base-stock level $x$, and $\s\in \st^x$, $g^x(\s)$ is convex in $x$.
\end{lemma}
\begin{proof}
Let $\s' := (x,0,\ldots,0)$ and 
let $\mean$ be the mean of demand distribution $F$. 
By Lemma \ref{lem:connection} we have that $g^x(\s')=\lambda^x-p \mu$. Therefore, under given assumption that demand distribution $F$ has a non-zero probability of zero demand, we can use Lemma \ref{lem:convex_ref} to conclude that the first term is convex in $x$, which implies $g^x(\s')$ is convex. Now, by Lemma \ref{lem:uniformLoss}, for any state $\s \in \st^x$, $g^x(\s)=g^x(\s')$. Therefore,  $g^x(\s)$ is convex in $x$ for all $\s\in \st^x$.
\end{proof}
\subsection{Bound on bias}
\label{sec:prop2}
Now, the following can be obtained as a corollary of Lemma \ref{lem:value} and the definition of bias.
\begin{lemma}[Bounded bias lemma]
\label{lem:biasBound}
For any $x$ and $\s, \s'\in \st^x$, 
\[v^{x}(\s) - v^{x}(\s') \leq 36\max(h,p)Lx.\]
\end{lemma}
\begin{proof}
From Lemma \ref{lem:uniformLoss}, $g^x(\s_t)=g^x(\s_t') = g^x$ for all $t$. Now by definition of $v^x(\cdot)$,
\[ \textstyle v^x(\s) = \Ex\left[\lim_{T\rightarrow \infty }\sum_{t=1}^T C^x(\s_t) - g^x | \s_1=\s \right] = \lim_{T\rightarrow \infty }V^x_T(\s) - T g^x\]
and \[\textstyle v^x(\s') = \Ex\left[\lim_{T\rightarrow \infty }\sum_{t=1}^T C^x(\s_t) - g^x | \s_1=\s' \right] = \lim_{T\rightarrow \infty }V^x_T(\s') - T g^x.\]
We note that both of the above limits exists (see Remark \ref{rem:finite}), and hence by Lemma \ref{lem:value}, 
\begin{eqnarray*}
v^x(\s)- v^x(\s') &=& \lim_{T\rightarrow \infty }(V^x_T(\s) - T g^x) - \lim_{T\rightarrow \infty }(V^x_T(\s') - T g^x) \\
&=& \lim_{T\rightarrow \infty} V^x_T(\s) - V^x_T(\s') \\
&\leq& 36\max(h,p)Lx.
\end{eqnarray*}
\end{proof}

\subsection{Concentration of finite horizon average cost}
\label{sec:prop3}
We use the following known relation between loss and bias which holds under finite state and action space.
\begin{lemma}[\cite{puterman2014markov}, Theorem 8.2.6] 
 \label{thm:gain_eq}
For any $\s\in \st^x$, the bias and loss satisfy the following equation:
\[g^x(\s) = C^x(\s) + \Ex_{\s'\sim P^x(\s)}[v^x(\s')] - v^x(\s).\] 
\end{lemma}

\begin{lemma}[Concentration lemma]
\label{cor:conc}
Given a base-stock level $x$, let $\gamma >0$ and $N=\frac{\log(T)}{\gamma^2}$. Then, for any $\s_1\in \st^x$, with probability $1-\frac{1}{T^2}$,
\[\left|\frac{1}{N}\sum_{t=1}^N C^x(\s_t) - g^x(\s_1) \right| \leq 108\max(h,p)Lx\gamma.\]
\end{lemma}
\begin{proof}
By Theorem \ref{thm:gain_eq}, the loss $g^x$ and bias $v^x$ satisfy: $g^x(\s) = C^x(\s) + \Ex_{\s'\sim P^x(\s)}[v^x(\s')] - v^x(\s)$ for all states $s\in \st^x$. Note that this equation continues to hold if any constant $c$ is added to all $v^x(\s)$. Therefore, for the purpose of using this equation, without loss of generality, we can assume that $\min_{\s \in \st^x} v^x(\s) = 0$, and from Lemma \ref{lem:biasBound} we have $0\le v^x(\s)\le 36\max(h,p)Lx$  for all $\s \in \st^x$.

Also, note that if $\s_1\in\st^x$ then all subsequent states $\s_t$ in MRP $\mrpx$ are in $\st^x$. Therefore, from Lemma \ref{lem:uniformLoss}, $g^x(\s_1)=g^x(\s_t)$ for all $t$. We use these observations to derive the following. 

\begin{eqnarray*}
\left|\left(\frac{1}{N}\sum_{t=1}^N C^x(\s_t)\right) - g^x(\s_1)\right| &=& \left|\frac{1}{N}\sum_{t=1}^N \left(C^x(\s_t)- g^x(\s_t)\right)\right| \\
&=&  \left|\frac{1}{N}\sum_{t=1}^N (C^x(\s_t)- (C^x(\s_t) + \Ex_{\s'\sim P^x(\s_t)}[v^x(\s')] - v^x(\s_t))\right| \\
&=&  \left|\frac{1}{N}\sum_{t=1}^N v^x(\s_t) - \Ex_{\s'\sim P^x(\s_t)}[v^x(\s')]\right| \\
&=&  \left|\frac{1}{N}(v^x(\s_1) - \Ex_{\s'\sim P^x(\s_N)}[v^x(\s')])\right.\\
&&+ \left.\frac{1}{N}\sum_{t=1}^{N-1} v^x(\s_{t+1}) - \Ex_{\s'\sim P^x(\s_t)}[v^x(\s')]\right| \\
&\leq& \frac{36\max(h,p)Lx}{N} + \left|\frac{1}{N}\sum_{t=1}^{N-1} v^x(\s_{t+1}) - \Ex_{\s_{t+1}\sim P^x(\s_t)}[v^x(\s_{t+1})]\right|.
\end{eqnarray*} 
Now, let \[\Delta_{t+1} := v^x(\s_{t+1}) - \Ex_{\s'\sim P^x(\s_t)}[v^x(\s')].\] Note that $\Ex[\Delta_{t+1} | s_t] = 0$ and hence $\Delta_t$'s form a martingale difference sequence with $|\Delta_t| \leq 36\max(h,p)Lx$ for all $t$ (the distribution $P^x(\s_t)$ is supported only on states in $\st^x$). Thus, we can apply Azuma-Hoeffding's inequality (Theorem \ref{azuma} in the Appendix) to show that for any $\epsilon >0$,  \[P(|\sum_{t=1}^{N-1} \Delta_t| \geq \epsilon) \leq 2exp(-\frac{\epsilon^2}{2(N-1){(36\max(h,p)Lx)}^2}).\] 

Therefore, by setting $\epsilon=72\max(h,p)Lx\sqrt{N\log(T)}$, we obtain that with probability at least $1-\frac{1}{T^2}$, 
\begin{eqnarray*}
|\frac{1}{N}\sum_{t=1}^N C^x(\s_t) - g^x(\s_1)| &\leq& \frac{36\max(h,p)Lx}{N} + \frac{1}{N} (72\max(h,p)Lx\sqrt{N\log(T)}) \\
&\leq& 108\max(h,p)Lx\sqrt{\frac{\log(T)}{N}} \\
\end{eqnarray*}
The result follows by substituting $N=\frac{\log(T)}{\gamma^2}$.
\end{proof}

\begin{remark}
Observe from the above lemmas that the bias, and the difference between expected total cost and asymptotic cost, are $0$ when lead time is $0$. 
\end{remark}

\section{Algorithm design}
\label{sec:algo}
We design a learning algorithm for  the inventory control problem when the demand distribution $F$ is a priori unknown. The algorithm seeks to minimize regret in the total expected cost compared to the asymptotic cost of the best base-stock policies in a pre-specified range $\range$ (refer to regret definition in \S\ref{sec:probdef}). The algorithm receives as input, the range  $\range$ of base-stock levels to compete with, the fixed delay parameter $L$, and the time horizon $T$, but not the demand distribution $F$. 

\mypara{Challenges and main ideas.} Our algorithm crucially utilizes the observations made in section \S\ref{sec:technical} regarding convexity of the average cost when a base-stock policy is used. Based on this observation, we utilize ideas from exploration-exploitation algorithms for stochastic convex bandits, in particular the algorithm in  \citep{agarwal2011stochastic} for $1$-dimensional stochastic convex bandits.

In the stochastic convex bandit problem, in every round the decision maker chooses a decision $x_t$ and observes a noisy realization of $f(x_t)$, where $f$ is some fixed but unknown convex function, and the noise is i.i.d. across rounds. The goal of an online algorithm is to use past observations to make decisions $x_t, t=1, \ldots, T$ in order to minimize the regret against the best single decision, i.e., minimize regret $\sum_{t=1}^T (f(x_t)-f(x^*))$ where $x^* = \arg\min_{x\in X} f(x)$. 
Therefore, based on the definition of regret in the inventory control problem, one may want to consider a mapping to the stochastic convex bandit problem by setting $f(x)$ as $\lambda^x$, the average cost of the base-stock policy with level $x$. 

However, there are several challenges in achieving this mapping. Firstly, the instantaneous holding cost and lost sales penalty depends on the current inventory state, and therefore is not a noisy realization of $f(x)=\lambda^x$ (more precisely, the noise is not i.i.d. across rounds). Further, a part of the instantaneous cost, that is, the lost sales penalty, is not even observed. 

We overcome these challenges using the construction of pseudo-cost and the concentration results derived in the previous section. In particular, in Lemma \ref{lem:connection} we proved that expected infinite horizon average pseudo-cost $g^x(\s)$ starting from a state $\s\in \st$ differs from the average true cost $\lambda^x$ by amount $p\mu$. Since this deviation of $p\mu$ is fixed and {\it does not depend on the policy used}, the following equivalence between regret in pseudo-cost vs. regret in true costs follows almost immediately.  

\begin{lemma}\label{lem:reg_equiv}
Recall $\trueCost_t=h(I_t-d_t)^+ + p(d_t-I_t)^+$ is the true cost in step $t$. Let $C_t=\trueCost_t-pd_t$  be the observed cost (i.e., pseudo-cost) at time $t$. Then,  regret under the true cost  is equivalent to the regret under the observed cost, i.e.,
\[\Reg(T) : = \Ex[\sum_{t=1}^T \trueCost_t] - T \left(\min_{x\in \range}\lambda^x\right) = \Ex[\sum_{t=1}^T C_t] - T \left(\min_{x\in \range} g^x\right). \]
where $g^x$ is the loss of MRP  $\mrpx$ starting in any state $\s\in \st^x$ (refer to Definition \ref{def:loss} and Lemma \ref{lem:uniformLoss}).
\end{lemma}
\begin{proof}
\begin{eqnarray*}
\Reg(T) &=& \textstyle \Ex[\sum_{t=1}^T \trueCost_t] - T \left(\min_{x\in \range}\lambda^x\right) \\
\text{(using Lemma \ref{lem:connection} and \ref{lem:uniformLoss})} &=& \textstyle \Ex[\sum_{t=1}^T C_t + pd_t] - T \left(\min_{x\in \range} g^x + p\mean \right) \\
\text{(using independent demand assumption)}&=& \textstyle \Ex[\sum_{t=1}^T C_t] - T \left(\min_{x\in \range}g^x\right).
\end{eqnarray*}
\end{proof}

Thus, we can focus on designing an algorithm for minimizing pseudo-costs. Further, the concentration results in Lemma \ref{cor:conc}, derived by bounding bias of base-stock policies, allow us to develop confidence intervals on estimates of cost functions in a manner similar to stochastic convex bandit algorithms.



\mypara{Algorithm description.} Our algorithm is derived from the algorithm in  \citep{agarwal2011stochastic} for $1$-dimensional stochastic convex bandits with convex function $f(x)=g^x$. Following are the main components of our algorithm.

\paragraph{\it Working interval of base-stock level:} Our algorithm maintains a high probability confidence interval that contains an optimal base-stock level. Initially, this is set as $\range$, the pre-specified range received as an input. As the algorithm progresses, the working interval is refined by discarding portions of this interval which have low probability of containing the optimal base-stock level. 

{\it Epoch and round structure:} Our algorithm proceeds in epochs ($k=1,2,\ldots$), a group of consecutive time steps where the same working interval of base-stock level is maintained throughout an epoch and denoted as $[l_k,r_k]$. Epochs are also further split up into groups of consecutive time steps called rounds. In round $i$ of epoch $k$, the algorithm first plays the policy $\pi^0$, which is to order $0$ in every time step, until the sum of total inventory and on hand orders falls below $x_l:=l_k + \frac{r_k-l_k}{4}$. Then, the algorithm plays policies $\pi^{x_l}, \pi^{x_c}, \pi^{x_r},$ denoting base stock policies corresponding to base-stock levels $x_l := l_k + \frac{r_k-l_k}{4}, x_c:=l_k + \frac{r_k-l_k}{2}, x_r:=l_k + \frac{3(r_k-l_k)}{4}$, respectively, for $N_i$ time steps each. 
Note that on executing base-stock policies in the given order, the algorithm always starts executing a base-stock policy $\pi^x$ for $x\in(x_l, x_c, x_r)$ 
at a total inventory position below $x$. Therefore, it will immediately (in one step) reach the desired inventory position $x$. 
Here 
\begin{center}
    $N_i=\log(T)/\gamma^2_i$ with $\gamma_i=2^{-i}$.
\end{center}
Therefore, the number of observations quadruples in each round. 
At the end of every round, these observations are used to update a confidence interval estimate for average cost. An epoch ends when the confidence intervals at the end of a round meet a certain condition, as defined next. 

{\it Updating confidence intervals.} Given a vector 
$\Cv_N = (C_1,C_2,...,C_N)$, define 
\begin{equation}
    \label{def:ci}
    LB(\Cv_N) := \frac{1}{N}\sum_{i=1}^N C_i - \frac{H\gamma}{2}, \text{ and } UB(\Cv_N) := \frac{1}{N}\sum_{i=1}^N C_i +\frac{H\gamma}{2},
\end{equation}
where $\gamma=\sqrt{\frac{\log(T)}{N}}$, $H:= \Hval$ and $L$ is the known lead time. 

Now, let $\Cv^\ell_N, \Cv^c_N, \Cv^r_N$ denote the $N=N_i$ realizations of pseudo-costs ($C^x_t$) observed on running base-stock policy $\pi^x$ for each of the three levels $x \in [x_l,x_c,x_r]$ in round $i$. Then, at the end of round $i$, the algorithm computes three intervals:
\begin{center}
    $[LB(\Cv^a_N), UB(\Cv_N^a)]$ for $a\in \{l, c, r\}$. 
\end{center}
Using the Lemma \ref{cor:conc} proven in the previous section to bound the difference between expected cost and expected asymptotic cost, and Lemma \ref{lem:obs_conc} to bound the difference between empirical cost and expected cost, we can show that (see Lemma \ref{lem:ci_prob}) the loss of each of these base-stock policies $g^{x_a} \in [LB(\Cv^a_N), UB(\Cv^a_N)]$ with probability $1-\frac{1}{T^2}$. Therefore, each of these intervals is a high confidence intervals for the respective loss, with endpoints determined by $N$ observed empirical costs. 

\begin{algorithm}[ht]
\caption{Learning algorithm for the inventory control problem}
\label{alg}
\begin{algorithmic}
\STATE {\bf Inputs:} Base-stock range $\range$, lead time $L$, time horizon $T$. 
\STATE {\bf Initialize:} $l_1 := 0$, $r_1 := \rmax$. 
\FOR{epochs $k=1, 2, \ldots, $} 
	\STATE Set $w_k := r_k-l_k$, the width of the working interval $[l_k, r_k]$.
	\STATE Set $x_l := l_k+w_k/4$, $x_c := l_k+w_k/2$, and $x_r := l_k+3w_k/4$. 
	\FOR{round $i=1, 2, \ldots, $} 
		\STATE Let $\gamma_i = 2^{-i}$ and $N = \frac{\log(T)}{\gamma_i^2}$.
		\STATE Play policy $\pi^{0}$ until a time step $t$ with inventory position $(\inv_t+o_{t-1}+\cdots + o_{t-L})\le x_l$.
		\STATE Play policy $\pi^{x_l}, \pi^{x_c}, \pi^{x_r}$, each for $N$ time steps to observe $N$ realizations of pseudo-costs ($C_t=\trueCost_t-pd_t$); store as vectors $\Cv^l_N, \Cv^c_N, \Cv^r_N$ respectively.
		\STATE If at any point during the above two steps, the total number of time steps reaches $T$, exit.
		\STATE For each $a \in \{l, c, r\}$, use $\Cv^a_N$  to calculate a confidence interval $[LB(\Cv_N^a), UB(\Cv_N^a)]$ of length $H\gamma_i$ as given by \eqref{def:ci}, where $H=\Hval$. 
		\IF {$\max\{LB(\Cv_N^l), LB(\Cv_N^r) \} \geq \min\{UB(\Cv_N^l), UB(\Cv_N^c), UB(\Cv_N^r)\} + H\gamma_i$} 
			\STATE \algorithmicif \ $LB(\Cv_N^l) \geq LB(\Cv_N^r)$ \algorithmicthen \ $l_{k+1} := x_l$ and $r_{k+1} = r_k$.
			\STATE \algorithmicif \ $LB(\Cv_N^l) < LB(\Cv_N^r)$ \algorithmicthen \ $l_{k+1} := l_k$ and $r_{k+1} = x_r$
			\STATE Go to next epoch $k+1$.
		\ELSE 
			\STATE Go to next round $i+1$.
		\ENDIF
	\ENDFOR
\ENDFOR
\end{algorithmic}
\end{algorithm}

At the end of every round $i$ of an epoch $k$, the algorithm uses the updated confidence intervals to check if either  the portion $[l_k, x_l]$ or the portion $[x_r, r_k]$ of the working interval $[l_k, r_k]$ can be eliminated. Given the confidence intervals, the test used for this purpose is exactly the same as in \cite{agarwal2011stochastic}, and uses convexity properties of function $g^x$. If the test succeeds, at least $1/4$ of the working interval is eliminated and the epoch $k$ ends. 

The algorithm is summarized as Algorithm \ref{alg}.



\section{Regret Analysis: Proof of Theorem \ref{thm:reg}} 
\label{sec:reg}



In this section, we prove the regret bound stated in Theorem \ref{thm:reg} for Algorithm \ref{alg}. Given the key technical results proven in section \S\ref{sec:technical}, the regret analysis follows steps similar to the regret analysis for stochastic convex bandits in \cite{agarwal2011stochastic}.  We use the notation $f(x)=g^x$ in this proof to connect the regret analysis here to the analysis for stochastic convex bandits with convex function $f$. Let $x^*=\min_{x\in \range} g^x = \min_{x\in \range} f(x)$. And, let $C_t$ be the cost (i.e., pseudo-cost) observed at time $t$. Then, by Lemma \ref{lem:reg_equiv},
\begin{eqnarray*}
\Reg(T)  
&=& \textstyle \Ex[\sum_{t=1}^T C_t] - \sum_{t=1}^T f(x^*).
\end{eqnarray*}
Also define an event $\event$ such that all confidence intervals $\ci$ calculated in Algorithm \ref{alg} satisfy: $g^{x_a} \in \ci$ for every epoch $k$, round $i$ and $a \in \{l,c,r\}$. The analysis in this section will be condition on $\event$, and the probability $P(\event)$ will be addressed at the end. 

We divide the regret in two parts: first we consider the regret over the set of times steps $T_{i,k,0}$ at the beginning of each epoch $k$ and round $i$ where policy $\pi^0$ is played until the leftover inventory depletes to a level below $x^l$. We denote the total contribution of regret from these steps (across all epochs and rounds) as $\Reg^0(T)$. Since the cost incurred at any time step is at most $\max(h,p)(\rmax)$, this part of the regret is bounded by
\begin{center}
$\Reg^0(T) \le \max(h,p)(\rmax) \cdot \Ex[\sum_{\text{epoch} k} \sum_{\text{round $i$ in epoch $k$}} |T_{k,i,0}|]$.
\end{center}

To bound expected number of steps in $T_{k,i,0}$, observe that ordering zero for $L$ steps will result in at most $\rmax$ inventory on hand and no orders in the pipeline. By definition of $D$, the expected number of time steps to deplete $\rmax$ units of inventory is upper bounded by $D\rmax$. Therefore, $\Ex[|T_{k,i,0}|] \leq L+D\rmax$. Since any epoch has at most $T$ time steps, and each successive round within an epoch has four times the number of time steps as the previous, there are at most $\log(T)$ rounds per epoch. Also, in Lemma \ref{lem:sb3} we show that, under $\event$, the number of epochs is bounded by
$\Kval.$ Intuitively, this holds because in every epoch we eliminate at least $(1/4)^{th}$ of the working interval.
Using these observations, the regret from all the time steps where policy $\pi^0$ was executed is bounded by
\begin{equation}\label{eq:reg1}
\Reg^0(T)\le \Kval\log(T)(L+D\rmax)\max(h,p)(\rmax).
\end{equation}

Next, we consider the regret over all remaining time steps, denoted as $\Reg^1(T)$. Algorithm \ref{alg} plays the base-stock policies with level $x_l, x_c$, or $x_r$ in these steps, where these levels are updated at the end of every epoch. 
Consider a round $i$ in epoch $k$. Let $T_{k,i,l}, T_{k,i,c}, T_{k,i,r}$, be the set of (at most) $N_i = \frac{\log(T)}{\gamma_i^2}$ consecutive times where policies $\pi^{x_l}, \pi^{x_c}, \pi^{x_r}$ are played, respectively, in round $i$ of epoch $k$. Here, $\gamma_i = 2^{-i}$, and recall that $H=\Hval$. Let $x_t$ denote the base-stock level used by the base-stock policy at time $t$. By Lemma \ref{cor:conc}, for epoch $k$, round $i$ and $a \in \{l,c,r\}$, under $\event$,
\[\left|\sum_{t\in T_{k,i,a}}  (C_t  - f(x_t)) \right| \leq  N_i H\gamma_i = \frac{H\log(T)}{\gamma_i}.\]

Substituting above, we can derive that
\begin{eqnarray}
\label{eq:reg1:1}
\Reg^1 &=& \Ex\left[\sum_{\text{epoch $k$ round $i$}} \sum_{a \in \{l,c,r\}} \sum_{t\in T_{k,i,a}} (C_t   -f(x^*)\right] \nonumber\\
&=& \Ex\left[\sum_{\text{epoch $k$ round $i$}} \sum_{a \in \{l,c,r\}} \sum_{t\in T_{k,i,a}} (C_t   -f(x_t) + f(x_t) -f(x^*))\right] \nonumber\\
& \le & \Ex\left[\sum_{\text{epoch $k$ round $i$}} \left(3 N_iH\gamma_i + \sum_{a \in \{l,c,r\}} \sum_{t\in T_{k,i,a}} (f(x_t)-f(x^*))\right)\right].
\end{eqnarray}

Now observe that for any round $i$ of epoch $k$ in which the algorithm does not terminate, the total number of time steps is bounded by $T$. So for $a \in \{l,c,r\}$, we have $|T_{k,i,a}| = \frac{\log(T)}{\gamma_i^2} \leq T$, which implies $\gamma_i \geq \sqrt{\frac{\log(T)}{T}}$. Since $\gamma_{i+1} = \frac{1}{2}\gamma_{i}$, let us define $\gamma_{min} := \frac{1}{2}\sqrt{\frac{\log(T)}{T}}$ so that $\gamma_{min} \leq \gamma_j$ for any round $j$. 
Recall that $\gamma_i = 2^{-i}$  so we can bound the geometric series: 
\begin{equation} \label{eq:geo_sum}
\sum_{k} \sum_{i} 3 N_iH\gamma_i = \sum_{k} \sum_i \left(\frac{3H\log(T)}{\gamma_i}\right) \le \Kval \left(\frac{3H\log(T)}{\gamma_{min}}\right).
\end{equation}

Substituting the value of $\gamma_{min}$ we get a bound of $6H\Kval\sqrt{T\log(T)}$ on the first term in \eqref{eq:reg1:1}. 
Now, consider the second term in \eqref{eq:reg1:1}. We use the results in \cite{agarwal2011stochastic} regarding the convergence of the convex optimization algorithm to bound the gap between $f(x_t)$ and $f(x^*)$. Intuitively, in every epoch the working interval shrinks by a constant factor, so that $x_t\in\{x_l, x_c, x_r\}$ are closer and closer to the optimal level $x^*$. Therefore, the gap  $|f(x_t)-f(x^*)|$ can be bounded using a Lipschitz property of $f$ proven in Lemma \ref{l_factor} that shows $|f(x_t)-f(x^*)|\le \max(h,p)|x_t-x^*|$. 
Specifically, we adapt the proof  from \cite{agarwal2011stochastic} to derive the following bound (details are in Lemma \ref{lem:sb4} in appendix):
\begin{equation}\label{eq:reg3}
\sum_{k,i, a, t\in T_{k,i,a}} f(x(t)) - f(x^*) \leq 146 H\Kval\sqrt{T\log(T)}.
\end{equation}
Substituting, in \eqref{eq:reg1:1}, $\Reg^1(T)$ is bounded by:
\begin{equation}
\Reg^1(T)\le 6H\Kval\sqrt{T\log(T)} + 146 H\Kval\sqrt{T\log(T)} = 152 H\Kval\sqrt{T\log(T)}.
\end{equation}
And, combining with the bound on $\Reg^0(T)$ from  (\ref{eq:reg1}), we get the following regret bound:
\begin{eqnarray*}
\Reg(T) & \le & \Kval\log(T)(L+D\rmax)\max(h,p)(\rmax) + 76H \Kval\sqrt{T\log(T)}\\
& = & \tilde{O}\left(D\max(h,p)\rmax^2 + (L+1)\max(h,p)\rmax\sqrt{T} \right).
\end{eqnarray*}
We complete the proof of the theorem statement by noting all the analysis has been conditioned on $\event$ where $g^{x_a} \in \ci$ for every epoch $k$, round $i$ and $a \in \{l,c,r\}$. By Lemma \ref{cor:conc}, the condition is satisfied with probability at least $1-\frac{1}{T^2}$ for each $k,i,a$. Since there are no more than $T$ time steps and therefore at most $T$ plays of any policy, by union bound \[P(\event) \geq 1-\frac{1}{T},\]
and hence the given regret bound holds with probability $1-\frac{1}{T}$.


Finally, to see that a similar regret bound holds for the alternative regret definition in Remark \ref{rem:alter}, we compare the two regret definitions:
\begin{eqnarray*}
\Reg'(T) & = & \Reg(T) +   T  \lambda^{x^*} - \Ex\left[\sum_{t=1}^T \bar{C}_t^{x^*} \right]  
\end{eqnarray*}
Now, use (\ref{eq:cost1}) and Lemma \ref{lem:connection} to convert true costs to pseudo-costs; and then apply Lemma \ref{cor:conc} to obtain
$$T  \lambda^{x^*} - \Ex\left[\sum_{t=1}^T \bar{C}_t^{x^*} \right] =  T g^{x^*}- \Ex\left[\sum_{t=1}^T C_t^{x^*} \right] \le O(H	\sqrt{T \log(T)})$$
Hence the two regret bounds are of the same order.

\section{Conclusions}
\label{sec:conc}
We presented an algorithm to minimize regret in the periodic inventory control problem under censored demand, lost sales, and positive lead time, when compared to the best base-stock policy. By using convexity properties of the long run average cost function and a newly proven bound on bias of base-stock policies, we extend a stochastic convex bandit algorithm to obtain a simple algorithm that substantially improves upon the existing solutions for this problem. In particular, the regret bound for our algorithm maintains an optimal dependence on $T$, while also achieving a linear dependence on the other problem parameters like lead time. The algorithm design and analysis techniques developed here may be useful for obtaining efficient solutions for other classes of learning problems where the MDPs involved may be large, but the long-run average cost under benchmark policies is convex.

\bibliographystyle{plain}
\bibliography{references}
\newpage

\appendix
\section{Concentration bounds}
\label{app:a}

\begin{theorem}[Azuma-Hoeffding inequality] \label{azuma}
 Let $X_1,X_2,\ldots$ be a martingale difference sequence with $|X_i| \leq c$ for all $i$. Then for all $\epsilon>0$ and $n \in \mathcal{N}$, \[P(|\sum_{i=1}^n X_i| \geq \epsilon) \leq 2exp(-\frac{\epsilon^2}{2nc^2}).\]
\end{theorem}

We note that instantaneous costs observed from running Algorithm \ref{alg} are different from the costs defined in the cost function of $\mrpx$. The following lemma gives a concentration bound on the $N$ step observed instantaneous costs and the expected $N$ step cost. In conjunction with Lemma \ref{cor:conc}, these two results are used to give a high probability confidence interval containing the true loss at base-stock level $x$, using only observed samples of the instantaneous cost.

\begin{lemma}[Concentration of $N$ observed cost and expected cost] \label{lem:obs_conc}
Given a base-stock level $x\in \range$, let $H:= \Hval$, $\gamma >0$ and $N\ge \frac{\log(T)}{\gamma^2}$. Then for any $\s_1\in \st^x$, with probability $1-\frac{1}{T^2}$, \[\frac{1}{N}|\sum_{t=1}^N C^x(\s_t) - \sum_{t=1}^N C^x_t| \leq \frac{H\gamma}{4}.\]
\end{lemma}
\begin{proof}
Let $\mathcal{F}_n$ be the filtration with respect to the states $\s_1, \s_2, \ldots, \s_{n}$. Define \[Y_{n-1} := \Ex[\sum_{t=1}^N C^x_t | \mathcal{F}_{n}].\] Then the sequence $Y_0,\ldots,Y_N$ forms a Doob martingale sequence with $Y_0 = \sum_{t=1}^N C^x(\s_t)$ and $Y_N = \sum_{t=1}^N C^x_t$. 
Furthermore for $n \in \{1,2,\ldots,N\}$, 
\begin{eqnarray*}
|Y_n-Y_{n-1}| &=& |C^x_{n-1} + V^x_{N-n+1}(s_n)-V^x_{N-n+2}(s_{n-1}))| \\
&=& |C^x_{n-1} + V^x_{N-n+1}(s_n)-(V^x_{N-n+1}(s_{n-1}) \\
&&+\Ex[C^x_N|\s_{n-1}])| \\
&\leq& 36\max(h,p)Lx + 2\max(h,p)x \\
&\leq& 36\max(h,p)(L+1)x
\end{eqnarray*}
using Lemma \ref{lem:value} to bound the difference in values and the fact that the maximum cost in a single time step with at most $x$ units of inventory in the pipeline is bounded above by $\max(h,p)x$ to bound one time instance difference in observed and expected cost.
Now, we can apply Azuma-Hoeffding's inequality (Theorem \ref{azuma}) to derive that for any $\epsilon >0$, \[P(|\sum_{i=1}^N Y_i-Y_{i-1}| \geq \epsilon) \leq 2exp(-\frac{\epsilon^2}{2N(36\max(h,p)Lx)^2}).\]
Therefore, by setting $\epsilon=72\max(h,p)Lx\sqrt{N\log(T)}$, we obtain that with probability at least $1-\frac{1}{T^2}$, \[\frac{1}{N}|Y_N-Y_0| = \frac{1}{N}|\sum_{t=1}^N C^x(\s_t) - \sum_{i=1}^N C^x_t| \leq \frac{H}{4}\sqrt{\frac{\log(T)}{N}}.\]
The result follow by substituting $N\ge \frac{\log(T)}{\gamma^2}$.
\end{proof}

\begin{lemma} \label{lem:ci_prob}
Let $\Cv_N$ be the vector formed by sequence of   observed (pseudo) costs on running base-stock policy with level $x$, with $N\ge \frac{\log(T)}{\gamma^2}$. Following the definition of $LB(\Cv_N),UB(\Cv_N)$ given by (\ref{def:ci}), with probability $1-\frac{1}{T^2}$, \[g^x \in [LB(\Cv_N),UB(\Cv_N)].\]
\end{lemma}
\begin{proof}
By Lemmas \ref{cor:conc} and \ref{lem:obs_conc} (noting that $C_i = C^x_i$), with probability $1-\frac{1}{T^2}$,
\[|\frac{1}{N}\sum_{i=1}^N C^x(\s_i) - g^x | \leq 108\max(h,p)Lx\gamma \leq \frac{H\gamma}{4} \text{ and } |\frac{1}{N}\sum_{i=1}^N C^x(\s_i) - \frac{1}{N}\sum_{i=1}^N C_i| \leq \frac{H\gamma}{4}.\]
Hence
\[|\frac{1}{N}\sum_{i=1}^N C_i - g^x | =  |\frac{1}{N}\sum_{i=1}^N C_i - \frac{1}{N}\sum_{i=1}^N C^x(\s_i) + \frac{1}{N}\sum_{i=1}^N C^x(\s_i) - g^x |\leq \frac{H\gamma}{2}\]
so \[ \frac{1}{N}\sum_{i=1}^N C_i - \frac{H\gamma}{2} \leq g^x \leq \frac{1}{N}\sum_{i=1}^N C_i +\frac{H\gamma}{2}.\]
The result follows noting from (\ref{def:ci}) that \[LB(\Cv_N) = \frac{1}{N}\sum_{i=1}^N C_i - \frac{H\gamma}{2} \text{ and } UB(\Cv_N) = \frac{1}{N}\sum_{i=1}^N C_i +\frac{H\gamma}{2}.\]

\end{proof}

\section{Proof details for Lemma \ref{lem:value}}
\label{app:b}

In this section we provide the proof details for results used in Lemma \ref{lem:value}, when $L\geq 1$. Recall that MRP $\mrpx$ is defined such that state $\s=(s(0),s(1),\ldots,s(L))$ with $s(0)$ being the on-hand inventory after the current time step's order arrival and new order, and $s(1),\ldots,s(L)$ are outstanding orders, with $s(L)$ being the most recent order, scheduled to arrive $L$ time steps from the current time. New orders are placed such that at every time step $t=1,2,\ldots,T$ we have $\sum_{i=0}^L \s_t(i)=x$, where $\s_t$ is the state at time $t$. We observe on-hand inventory level $I_t := \s_t(0)$ and sales given by $y_t := \min(d_t, I_t)$. The sales $y_t$ also happens to be the order placed in the next time step (the first order $y_1$ is a bit different, it is such that state $\s_1$ has total inventory level $x$). The new state at time $t+1$ is given by \[\s_{t+1} = (s_t(0) - y_t + s_t(1), s_t(2),\ldots,s_t(L),y_t).\] 
Let $n_T(\s_1):= \sum_{t=1}^T y_t$ denote the sum of sales from time $1$ to $T$, and $m_T(\s_1):= \sum_{t=1}^T I_t$ the sum of on-hand inventory levels.

\subsection{Bounding cumulative observed sales}
We bound the difference between the total sales in time $T$ starting from two different states $\s,\s'$ when the states satisfy the following property given below.
\begin{definition} \label{def:state_succ}
Define states $\s:=(s(0),s(1),\ldots,s(L)), \s':=(s'(0),s'(1),\ldots,s'(L))$. We say that $s' \succeq s$ if $\s'=(s(0)+\delta_0,s(1)+\delta_1,\ldots,s(L)+\delta_L)$ where $\delta_0+\delta_1+\ldots+\delta_L =0$ and there exists some $0\leq k \leq L-1$ such that $\delta_i \geq 0$ for all $i \in \{0,1,\ldots,k\}$ and $\delta_i \leq 0$ for all $i \in \{k+1,k+2,\ldots,L\}$.
\end{definition}

We first provide a simple bound on $n_T(s'_1) - n_T(\s_1)$ when $\s_1' \succeq \s_1$ and $T\leq L+1$ which will be useful in our proof for larger $T$.

\begin{lemma}\label{lem:y_diff}
Assume $\s_1' \succeq \s_1$ and define $Y_t:= \sum_{i=1}^t y_t$, $Y_t':= \sum_{i=1}^t y_t'$ to be the total observed sales up to time $t$ starting from state $\s_1,\s_1'$, respectively. Then for $t=1,2,\ldots,L+1$, we have that \[Y_t'-Y_t \leq \max_{0 \leq k \leq t-1}(\delta_0 + \ldots + \delta_k).\]
\end{lemma}

\begin{proof}
We prove this statement by induction on $t$. For $t=1$, \[y_1'-y_1 = \min(s(0) + \delta_0, d_1)-\min(s(0),d_1) \leq \delta_0,\] since $\s_1' \succeq \s_1$ implies that $s(0) + \delta_0 \geq s(0)$. Assume for any time up to $t-1$ the hypothesis holds. Then, consider time $t$ and observe that:
\[I_t' = s_t'(0) = (s(0)+\delta_0+s(1)+\delta_1+\ldots+s(t-1)+\delta_{t-1})-(y_1'+\ldots+y_{t-1}')\]
and
\[I_t = s_t(0) = (s(0)+s(1)+\ldots+s(t-1))-(y_1+\ldots+y_{t-1}),\]
so subtracting we get
\begin{equation} \label{eqn_y}
I_t'-I_t +Y_{t-1}'-Y_{t-1} = \delta_0+\delta_1+\ldots+\delta_{t-1}.
\end{equation}
Now, we write
\[Y_{t}'- Y_t = y_t'-y_t  + Y_{t-1}'-Y_{t-1} = \min(I_t', d_t) - \min(I_t, d_t) + Y_{t-1}'-Y_{t-1}.\]
There are four cases to consider:
\begin{enumerate}
\item $d_t \leq I_t', d_t\leq I_t$: In this case $Y_{t}'- Y_t = d_t-d_t  + Y_{t-1}'-Y_{t-1} = Y_{t-1}'-Y_{t-1}\leq \max_{0 \leq k \leq t-2}(\delta_0 + \ldots + \delta_k) \leq \max_{0 \leq k \leq t-1}(\delta_0 + \ldots + \delta_k)$ by the induction hypothesis.

\item $d_t \geq I_t', d_t\geq I_t$: In this case $Y_{t}'- Y_t = I_t'-I_t  + Y_{t-1}'-Y_{t-1} = \delta_0+\ldots+\delta_{t-1} \leq \max_{0 \leq k \leq t-1}(\delta_0 + \ldots + \delta_k)$ by (\ref{eqn_y}).

\item $I_t \leq d_t \leq I_t'$: In this case $Y_{t}'- Y_t = d_t-I_t  + Y_{t-1}'-Y_{t-1} =d_t-I_t'+I_t'-I_t  + Y_{t-1}'-Y_{t-1} = d_t-I_t' +  \delta_0+\ldots+\delta_{t-1} \leq  \delta_0+\ldots+\delta_{t-1}\leq \max_{0 \leq k \leq t-1}(\delta_0 + \ldots + \delta_k)$ by (\ref{eqn_y}).

\item $I_t' \leq d_t \leq I_t$: In this case $Y_{t}'- Y_t = I_t'-d_t  + Y_{t-1}'-Y_{t-1} \leq Y_{t-1}'-Y_{t-1}\leq \max_{0 \leq k \leq t-2}(\delta_0 + \ldots + \delta_k) \leq \max_{0 \leq k \leq t-1}(\delta_0 + \ldots + \delta_k)$ by the induction hypothesis.
\end{enumerate}
Therefore, we have proven that under the induction hypothesis \[Y_{t}'- Y_t \leq \max_{0 \leq k \leq t-1}(\delta_0 + \ldots + \delta_k)\] and the desired result for all $t \in \{1,2,\ldots,L+1\}$ follows by induction.
\end{proof}

\begin{lemma}\label{lem:con_inv}
Consider the MRPs on following base-stock policy with level $x$ starting in states $\s_1, \s_1'\in \st^x$ with $\s_1' \succeq \s_1$. 
Let $I_t=\s_t(0), I_t'=\s_t'(0)$ be the on-hand inventory levels in the two processes at time $t$. Then, if $I_t'-I_t \geq 0$ for all $t \in \{1,2,\ldots,L+1\}$, then it holds that $n_T(\s_{L+1}') = n_T(\s_{L+1})$ for any $T$.
\end{lemma}
\begin{proof}
If we have $I_t'-I_t \geq 0$  then the respective sales at time $t$ satisfy $y_t'\geq y_t$ as well. Therefore, each entry of state $\s_{L+1}' = (I_{L+1}', y_1',y_2',\ldots,y_L')$ is at least the respective entry of state $\s_{L+1} = (I_{L+1}, y_1,y_2,\ldots,y_L)$. Since the total sum of the entries in each state is equal to $x$, we conclude that $\s_{L+1}' =\s_{L+1}$ and hence $n_T(\s_{L+1}') = n_T(\s_{L+1})$ for any $T$. 
\end{proof}

Above lemma shows that if we ever observe a $t$ with $\s_t' \succeq \s_t$ and the next $L$ consecutive on-hand inventory levels are at least as high in the process starting from state $\s_t'$ compared to starting from state $\s_t$, then the two processes will reach an identical state at time $t+L$ and hence all future observed sales will be the same. Utilizing this property, for states $\s_1' \succeq \s_1$ we can define the following sequence of times:

\begin{definition} \label{def:ts}
Given starting  states $\s_1' \succeq \s_1$, define a sequence of times \[1 = \sigma_0 < \tau_1 < \sigma_1 < \tau_2 < \sigma_2 < \ldots \leq \Gamma \] such that for $i\geq 1$, $\tau_i$ is the first time after $t=\sigma_{i-1}$ at which $I_{\tau_i}' < I_{\tau_i}$, $\sigma_i$ is the first time after $t=\tau_{i}$ at which $I_{\sigma_i}' > I_{\sigma_i}$, and $\Gamma$ is the first time at which $\s'_{\Gamma} = \s_{\Gamma}$.  By the previous lemma,  $\tau_i-\sigma_{i-1} \leq L+1$ and $\sigma_i-\tau_i \leq L+1$ (whenever $\tau_i,\sigma_i$ exist).
\end{definition}

\begin{lemma}
Given starting  states $\s_1' \succeq \s_1$ and the sequence defined above, $\s_{\sigma_i}' \succeq \s_{\sigma_i}$ and $\s_{\tau_i}' \preceq \s_{\tau_i}$ for all $i$, where $\tau_i,\sigma_i \leq \Gamma$.
\end{lemma}

\begin{proof}
We have $\s_{\sigma_0}' \succeq \s_{\sigma_0}$ is the starting state at time $t=1=\sigma_0$. If time $t=\tau_1$ exists, then $\tau_1 - \sigma_0 \leq L+1$ by Lemma \ref{lem:con_inv}. Furthermore, we can show that $\s_{\tau_1}' \preceq \s_{\tau_1}$. To see this, \[\s_{\tau_1}' = (I_{\tau_1}', s_{\sigma_0}'(\tau_1-\sigma_0), s_{\sigma_0}'(\tau_1+1-\sigma_0), \ldots, s_{\sigma_0}'(L), y'_{\sigma_0},\ldots,y'_{\tau_1-1})\] and \[\s_{\tau_1} = (I_{\tau_1}, s_{\sigma_0}(\tau_1-\sigma_0), s_{\sigma_0}(\tau_1+1-\sigma_0), \ldots, s_{\sigma_0}(L), y_{\sigma_0},\ldots,y_{\tau_1-1}).\]

By definition of $\tau_1$, for times $t \in \{\sigma_0, \sigma_0+1, \ldots, \tau_1-1\}$ we have $I'_t \geq I_t$ and hence $y_t' \geq y_t$. We also know that $I_{\tau_1}' < I_{\tau_1}$. It suffices to show that $s'_{\sigma_0}(i) \leq s_{\sigma_0}(i)$ for all $i \in \{\tau_1-\sigma_0, \tau_1+1-\sigma_0, \ldots, L\}$.

Recall $I_{\tau_1}' = I_{\tau_1-1}' - y_{\tau_1-1}' + s'_{\sigma_0}(\tau_1-1-\sigma_0)$ and $I_{\tau_1} = I_{\tau_1-1} - y_{\tau_1-1} + s_{\sigma_0}(\tau_1-1-\sigma_0)$ so that:
\[I_{\tau_1-1}' - y_{\tau_1-1}' + s'_{\sigma_0}(\tau_1-1-\sigma_0) < I_{\tau_1-1} - y_{\tau_1-1} + s_{\sigma_0}(\tau_1-1-\sigma_0) \leq I_{\tau_1-1}' - y_{\tau_1-1}'+ s_{\sigma_0}(\tau_1-1-\sigma_0)\] 
where the last inequality is because $I_{\tau_1-1}' \geq I_{\tau_1-1}$ implies (for any demand $d_{\tau_1-1}$), \[I_{\tau_1-1}'- y_{\tau_1-1}' = I_{\tau_1-1}'-\min(I_{\tau_1-1}', d_{\tau_1-1}) \geq I_{\tau_1-1}-\min(I_{\tau_1-1}, d_{\tau_1-1}) =  I_{\tau_1-1} - y_{\tau_1-1}.\] Hence, $s'_{\sigma_0}(\tau_1-1-\sigma_0) < s_{\sigma_0}(\tau_1-1-\sigma_0)$ and because $\s_1' \succeq \s_1$, $s'_{\sigma_0}(i) \leq s_{\sigma_0}(i)$ holds for all $i \in \{\tau_1-\sigma_0, \tau_1+1-\sigma_0, \ldots, L\}$. So we have shown that $\s_{\tau_1}' \preceq \s_{\tau_1}$. 

We can inductively apply the above argument for each successive $\sigma_i,\tau_i$, so that $\s_{\sigma_i}' \succeq \s_{\sigma_i}$ and $\s_{\tau_i}' \preceq \s_{\tau_i}$ for all $i$.

\end{proof}

Finally, we are ready to bound the difference in total observed sales in time $T$ between two states $\s_1' \succeq \s_1$ under policy $\pi^x$.

\begin{lemma}\label{lem:tot_demand}
Let $\s_1', \s_1\in \st^x$, and  $\s_1' \succeq \s_1$.  Then,

\[ |n_T^x(\s_1') - n_T^x(\s_1)| \leq 3x.\]

\end{lemma}
\begin{proof}
Let sequence $\sigma_0,\tau_1,\sigma_1,\ldots$ be as in Definition \ref{def:ts}.
First we show that \[n_T^x(\s_1') - n_T^x(\s_1) \geq -2x.\]
Let us assume that in our sequence of times the last $\sigma$ is $\sigma_M$. Then note that 
\begin{eqnarray*}
n_T^x(\s_1') - n_T^x(\s_1) &=& \sum_{t=1}^T y'_t - \sum_{t=1}^T y_t \\
&=& \sum_{i=0}^{M-1} \left(\sum_{j=\sigma_i}^{\sigma_{i+1}-1} (y'_j - y_j)\right) +  \sum_{t=\sigma_M}^T (y'_t - y_t).\\
\end{eqnarray*}

We will show that $\sum_{j=\sigma_i}^{\sigma_{i+1}-1} (y'_j - y_j) \geq 0$ for any $i=0,1,\ldots,M-1$. 
Consider the process starting from states $\s_{\sigma_i}', \s_{\sigma_i}$, where $\s_{\sigma_i}' \succeq \s_{\sigma_i}$ by the previous lemma. 

By (\ref{eqn_y}) in Lemma \ref{lem:y_diff}, 
\begin{eqnarray*}
&&(y'_{\sigma_i}+\ldots+y'_{\tau_{i+1}-1}) - (y_{\sigma_i}+\ldots+y_{\tau_{i+1}-1}) \\
&=& [(s'_{\sigma_i}(0) - s_{\sigma_i}(0)) +\ldots+(s'_{\sigma_i}(\tau_{i+1}-1-\sigma_i) - s_{\sigma_i}(\tau_{i+1}-1-\sigma_i))] -(I_{\tau_{i+1}}'-I_{\tau_{i+1}}).\\
\end{eqnarray*}

Now consider the process starting from states $\s_{\tau_{i+1}}' \preceq \s_{\tau_{i+1}}$. Recall that \[\s_{\tau_{i+1}}' = (I_{\tau_{i+1}}', s_{\sigma_i}'(\tau_{i+1}-\sigma_i), s_{\sigma_i}'(\tau_{i+1}+1-\sigma_i), \ldots, s_{\sigma_i}'(L), y'_{\sigma_i},\ldots,y'_{\tau_{i+1}-1})\] and \[\s_{\tau_{i+1}} = (I_{\tau_{i+1}}, s_{\sigma_i}(\tau_{i+1}-\sigma_i), s_{\sigma_i}(\tau_{i+1}+1-\sigma_i), \ldots, s_{\sigma_i}(L), y_{\sigma_i},\ldots,y_{\tau_{i+1}-1}),\] and as proved in the previous lemma $I_{\tau_{i+1}}' < I_{\tau_{i+1}}$, $s_{\sigma'_i}(i) \leq s_{\sigma_i}(i)$ for all $i \in \{\tau_{i+1}-\sigma_i,\ldots,L\}$, and $y_t' \geq y_t$ for all $t \in \{\sigma_i,\ldots,\tau_{i+1}-1\}$. So we have by Lemma \ref{lem:y_diff} that 

\begin{eqnarray*}
&& (y_{\tau_{i+1}}+\ldots+y_{\sigma_{i+1}-1}) - (y'_{\tau_{i+1}}+\ldots+y'_{\sigma_{i+1}-1}) \\
&\leq& (I_{\tau_{i+1}}-I'_{\tau_{i+1}}) + [(s_{\sigma_{i}}(\tau_{i+1}-\sigma_i) - s'_{\sigma_{i}}(\tau_{i+1}-\sigma_i))+\ldots + (s_{\sigma_{i}}(L) - s'_{\sigma_{i}}(L))] \\
&=& (I_{\tau_{i+1}}-I'_{\tau_{i+1}}) + [(s'_{\sigma_i}(0) - s_{\sigma_i}(0)) +\ldots+(s'_{\sigma_i}(\tau_{i+1}-1-\sigma_i) - s_{\sigma_i}(\tau_{i+1}-1-\sigma_i))]
\end{eqnarray*}
where the last equality follows from the fact that the sum of the entries in the states is always the same.

Combining the two results, we have that for any $i=0,1,\ldots,M-1$, \[\sum_{j=\sigma_i}^{\sigma_{i+1}-1} (y'_j - y_j) = (y'_{\sigma_{i}}+\ldots+y'_{\sigma_{i+1}-1}) - (y_{\sigma_{i}}+\ldots+y_{\sigma_{i+1}-1}) \geq 0.\] 

Therefore, we can conclude that \[n_T^x(\s_1') - n_T^x(\s_1) \geq \sum_{t=\sigma_M}^T (y'_t - y_t) = \sum_{t=\sigma_M}^{\hat{\Gamma}} (y'_t - y_t),\] where $\hat{\Gamma} := \min(\Gamma, T)$. By our construction of the $\sigma, \tau$ sequence,  $\hat{\Gamma} -\sigma_M +1 \leq 2(L+1)$ . Note that over any $L+1$ consecutive time steps, the total observed sales difference in those $L+1$ times  is at most $x$ for any two starting states. So $n_T^x(\s_1') - n_T^x(\s_1) \geq \sum_{t=\tau_M}^{\hat{\Gamma}} (y'_t - y_t) \geq -2x$.

To complete the proof, we show in a similar way that $n_T^x(\s_1') - n_T^x(\s_1) \leq 3x.$ Let us assume that in our sequence of times the last $\tau$ is $\tau_N$. Then note that 
\begin{eqnarray*}
n_T^x(\s_1') - n_T^x(\s_1) &=& \sum_{t=1}^T y'_t - \sum_{t=1}^T y_t \\
&=& \sum_{t=1}^{\tau_1-1} (y'_t - y_t) + \sum_{i=1}^{N-1} \left(\sum_{j=\tau_i}^{\tau_{i+1}-1} (y'_j - y_j)\right) +  \sum_{t=\tau_{N}}^T (y'_t - y_t).\\
\end{eqnarray*}

For any $i=1,2,\ldots,N-1$, consider the process starting from states $\s_{\tau_i}', \s_{\tau_i}$, where $\s_{\tau_i}' \preceq \s_{\tau_i}$ by the previous lemma. By an identical argument as above, $\sum_{j=\tau_i}^{\tau_{i+1}-1} (y'_j - y_j) \leq 0$, and $\sum_{t=\tau_N}^T (y'_t - y_t) = \sum_{t=\tau_N}^{\hat{\Gamma}} (y'_t - y_t) \leq 2x$. Noting that there are at most $L+1$ time steps in $\sum_{t=1}^{\tau_1-1} (y'_t - y_t)$, it is bounded by $x$, so we have shown that $n_T^x(\s_1') - n_T^x(\s_1) \leq 3x$ and hence with the other result, \[|n_T^x(\s_1') - n_T^x(\s_1)| \leq 3x.\]

\end{proof}

\subsection{Bounding cumulative on-hand inventory level}
\begin{lemma} \label{lem:tot_inv} 
Let $\s', \s\in \st^x$, and  $\s' \succeq \s$. Then,
\[ |m_T^x(\s) - m_T^x(\s')| \leq 6Lx.\]

\end{lemma}
\begin{proof}
Recall from before $y_t, I_t$ is the observed sales and on-hand inventory level at the beginning of time $t \geq 1$, respectively. Under base-stock level $x$ policy, the order placed at time $t$ is precisely $y_t$ (assume without loss of generality $y_1=0$, that is, we start at a state with total inventory level $x$). Also note that with starting state $\s = (s(0),s(1),\ldots,s(L))$ it makes sense to denote $y_0 = S(L), y_{-1} = s(L-2),\ldots, y_{1-L} = s(1)$. The on-hand inventory level transitions as follows: \[I_{t+1} = I_t - y_t + y_{t-L}.\]
Therefore, we can write that the inventory level at some time $k\geq 1$:
\[I_k = I_1 + \sum_{j=1}^{k-1}(y_{j-L} - y_j) \]
and hence the total sum of all inventory levels to time $T$ is:
\begin{eqnarray*}
\sum_{k=1}^T I_k &=& \sum_{k=1}^T( I_1 + \sum_{j=1}^{k-1}(y_{j-L} - y_j)) \\
&=& TI_1 + \sum_{k=2}^T\sum_{j=1}^{k-1}(y_{j-L} - y_j)\\
&=& TI_1 +\sum_{i=1}^{T-1}(T-i)(y_{i-L} - y_i) \\
&=& TI_1 +\sum_{i=1}^{T-1}(T-i)y_{i-L} - \sum_{i=1}^{T-1}(T-i)y_i.
\end{eqnarray*}
Now, if we break up the summations on the right hand side and reindex,
\[\sum_{i=1}^{T-1}(T-i)y_{i-L} = \sum_{i=1}^{L}(T-i)y_{i-L} + \sum_{i=L+1}^{T-1}(T-i)y_{i-L} = \sum_{i=1}^{L}(T-i)y_{i-L} + \sum_{i=1}^{T-L-1}(T-L-i)y_{i}\] 
and 
\[\sum_{i=1}^{T-1}(T-i)y_i = \sum_{i=1}^{T-L-1}(T-i)y_i + \sum_{i=T-L}^{T-1} (T-i) y_i,\]
so that:
\begin{eqnarray*}
\sum_{k=1}^T I_k &=& TI_1 +\sum_{i=1}^{T-1}(T-i)y_{i-L} - \sum_{i=1}^{T-1}(T-i)y_i \\
&=& TI_1 + \sum_{i=1}^{L}(T-i)y_{i-L} - \left(\sum_{i=1}^{T-L-1}Ly_i + \sum_{i=T-L}^{T-1}(T-i)y_i \right). \\
\end{eqnarray*}
Define \[\epsilon:= \sum_{i=1}^{T-1}Ly_i -(\sum_{i=1}^{T-L-1}Ly_i + \sum_{i=T-L}^{T-1}(T-i)y_i) \geq 0\]
and observe that \[\epsilon = \sum_{i=1}^{L-1} i y_{T-L+i} \leq L\sum_{i=1}^{L-1}y_{T-L+i} \leq Lx \] because in $L-1$ consecutive time steps, the total sales following $\pi^x$ cannot exceed $x$.

We can write
\begin{eqnarray*}
\sum_{k=1}^T I_k &=& TI_1 + \sum_{i=1}^{L}(T-i)y_{i-L} -\sum_{i=1}^{T-1}Ly_i +\epsilon \\
&=& TI_1 + \sum_{i=1}^{L}(T-i)y_{i-L} -\sum_{i=1}^{T-1}Ly_i +\epsilon \\
&=& \sum_{i=0}^{L}(T-i)s(i) -L\sum_{i=1}^{T-1}y_i +\epsilon  \\
\end{eqnarray*}
from the known values of $y_0,\ldots,y_{1-L}$.

Now let $I'_t,y'_t,s'(i), \epsilon'$ be the respective values if the starting state is $\s'$ instead of $\s$, and that $\s' \succeq \s$. By above, the difference $|m_x^T(\s')-m_x^T(\s)|$ can be bounded as
\begin{eqnarray*}
|m_x^T(\s')-m_x^T(\s)|  &=& \left| \sum_{k=1}^T I'_k -\sum_{k=1}^T I_k \right| \\
&\leq& \left|\sum_{i=0}^{L}(T-i)s'(i) - \sum_{i=0}^{L}(T-i)s(i)\right| + L\left|\sum_{i=1}^{T-1}(y'_i - y_i)\right| + |\epsilon'-\epsilon|\\
&\leq & (Tx-(T-L)x) + L(3x) + 2(Lx) \\
&=& 6Lx
\end{eqnarray*}
where we bound $|\sum_{i=0}^{L}(T-i)s'(i) - \sum_{i=0}^{L}(T-i)s(i)|$ by the largest possible value that occurs when $\s' = (x,0,\ldots,0)$ and $\s=(0,\ldots,0,x)$, and use Lemma \ref{lem:tot_demand} to bound $\sum_{i=1}^{T-1}(y'_i - y_i)$.
\end{proof}

\section{Proof details for Theorem \ref{thm:reg}}
\label{app:c}
Below we present additional results required for Theorem \ref{thm:reg}. Recall that $f(x):= g^x$ is a convex function and our confidence intervals are defined as in (\ref{def:ci}). Also recall that $\event$ is the event when all confidence intervals $[LB(C_N^a), UB(C^a_N)]$ calculated in Algorithm \ref{alg} satisfy: $g^{x_a} \in [LB(C_N^a), UB(C^a_N)]$ for every epoch $k$, round $i$ and $a \in \{l,c,r\}$.


\begin{lemma} \label{l_factor}
For $f(x) := g^x$ and $x \in \range$, the Lipschitz factor of $f(x)$ is $\max(h,p)$. That is, for $\delta\geq 0$, \[|f(x+\delta)-f(x)| \leq \max(h,p)\delta.\]
\end{lemma}
\begin{proof}
Let us compare the loss $g^{x+\delta}$ vs. $g^{x}$ on executing base-stock policy with level $x+\delta$ vs. $x$. Let us assume the starting state for the two MRPs are $\s_1^1 = (x+\delta,0,\ldots,0)$ and $\s_1^2=(x, 0, \ldots, 0)$ respectively (recall from Lemma \ref{lem:uniformLoss} that loss is independent of the starting state). We compare the two losses by coupling the execution of the two MRPs. For every time $t$, let $\s_t^1 := (I_t^1, o_{t-L+1}^1, \ldots, o_t^1)$ be the state of the system following policy with level $x+\delta$, and $\s_t^2 := (I_t^2, o_{t-L+1}^2, \ldots, o_t^2)$ be the state of the system following policy with level $x$. Define $\s_1^1 \geq \s_1^2$ if every entry in $\s_1$ is at least the respective entry in $\s_2$.

We will first show by induction that at each time step $t$, $\s_t^1 \geq \s_t^2$. In the first time step, we have $\s_1^1 = (x+\delta, 0, \ldots, 0) \ge (x, 0, \ldots, 0) = \s_1^2 $. 
From then on, the new order placed in time $t+1$ is the amount of sales in the previous time $t$. Therefore if at time $t$ we have that $\s_t^1 \geq \s_t^2$, then the orders at time $t+1$ satisfy $o_{t+1}^1 =  \min(d_t, I_t^1) \geq  \min(d_t, I_t^2) = o_{t+1}^2$. Also, $I_{t+1}^1 = (I_t^1 - \min(d_t, I_t^1)) + o_{t-L}^1 \geq (I_t^2 - \min(d_t, I_t^2)) + o_{t-L}^2  = I_{t+1}^2$. Hence we have $\s_{t+1}^1 \geq \s_{t+1}^2$. By induction, we have that for every $t\geq 1$, $\s_t^1 \geq \s_t^2$. 

We complete the proof by noting that additionally, at every time $t$, the total sum of the entries of $\s_t^1$ is exactly $\delta$ greater than the sum of the entries of $\s_t^2$. Therefore, the difference $0\le I_t^1 - I_t^2\le \delta$ for every $t$, which implies the difference in sales $0\le y_t^1-y_t^2 = \min(d_t, I_t^1) - \min(d_t, I_t^2)\le \delta$. Also 
$0\le (I_t^1-y_t^1)-(I_t^2-y_t^2) \le \delta$. Recall pseudo-cost $C^{x+\delta}_t=(I_t^1-y_t^1)h - py_t^1$, and   $C^{x}_t=(I_t^2-y_t^2)h - py_t^2$,
therefore,
we have for every $t$ and every sequence of demand realizations, 
\[|C^{x+\delta}_t - C^x_t| \leq  \max(h,p)\delta.\] 
By definition of loss $f(x)=g^x$ as average of pseudo-costs (see  Definition \ref{def:loss}), we have \[|f(x+\delta) - f(x)| 
\leq \max(h,p)\delta.\]
\end{proof}

The proofs for the remaining lemmas provided below are similar to the proofs of the corresponding lemmas in \cite{agarwal2011stochastic}. We include the proofs here for completeness.

\begin{lemma}  [Lemma 1 in \cite{agarwal2011stochastic}] \label{lem:sb1} Recall $[l_k,r_k]$ denotes the working interval in epoch $k$ of Algorithm \ref{alg}, with $[l_1, r_1]:=\range$. Then, under event $\event$, for epoch $k$ ending in round $i$, the working interval $[l_{k+1}, r_{k+1}]$ for the next epoch ${k+1}$ contains every $x \in [l_k,r_k] $ such that $f(x) \leq f(x^*) + H\gamma_i$, where $H = \Hval$. In particular, $x^* \in [l_k , r_k]$ for all epochs $k$.
\end{lemma}
\begin{proof}
Under Algorithm \ref{alg}, the epoch $k$ ends in round $i$ because \[\max\{LB(C_N^l), LB(C_N^r) \} \geq \min\{UB(C_N^l), UB(C_N^c), UB(C_N^r)\} + H\gamma_i. \] Hence either: 
\begin{enumerate}
    \item $LB(C^l_N) \geq UB(C^r_N) + H\gamma_i$,
    \item $LB(C^r_N) \geq UB(C^l_N) + H\gamma_i$, or
    \item $\max\{LB(C_N^l), LB(C_N^r) \} \geq UB(C^c_N) + H\gamma_i$.
\end{enumerate}
Consider the case (1) (case (2) is analogous). Then,
\[f(x_l) \geq f(x_r) + H\gamma_i.\] We need to show that every $x \in [l_k, l_{k+1}]$ has $f(x) \geq f(x^*) + H\gamma_i$. Pick $x\in [l_k, x_l]$ so that $x_l \in [x,x_r]$. Then $x_l = tx + (1-t)x_r$ for some $0\leq t\leq1$ so by convexity \[f(x_l)\leq tf(x) + (1-t)f(x_r).\] This implies that \[f(x) \geq f(x_r) + \frac{f(x_l)-f(x_r)}{t} \geq f(x_r) + \frac{H\gamma_i}{t} \geq f(x^*) + H\gamma_i,\] where we used that $t\leq 1$.

Now consider the case (3). Assume without loss of generality that $LB(C_N^l) \geq LB(C_N^r)$. Then we have \[f(x_l) \geq f(x_c) + H\gamma_i.\] We need to show that every $x \in [l_k, l_{k+1}]$ has $f(x) \geq f(x^*) + H\gamma_i$. This follows from the same argument as above with $x_r$ replaced by $x_c$. 
The fact that $x^* \in [l_k, r_k]$ for all epochs $k$ follows by induction.
\end{proof}

\begin{lemma} [Lemma 2 in \cite{agarwal2011stochastic}] \label{lem:sb2}
Under $\event$, if epoch $k$ does not end in round $i$, then $f(x) \leq f(x^*) +12H\gamma_i$ for each $x \in \{x_r,x_c,x_l\}$.
\end{lemma}
\begin{proof}
Under Algorithm \ref{alg}, round $i$ continues to round $i+1$ if \[\max\{LB(C_N^l), LB(C_N^r) \} < \min\{UB(C_N^l), UB(C_N^c), UB(C_N^r)\} + H\gamma_i. \]
We observe that since each confidence interval is of length $H\gamma_i$, this means that $f(x_l),f(x_c),f(x_r)$ are contained in an interval of length at most $3H\gamma_i$. By Lemma \ref{lem:sb1}, $x^* \in [l_k,r_k]$. Without loss of generality, assume $x^* \leq x_c$. Then there exists $t \geq 0$ such that $x^* = x_c + t(x_c-x_r)$, so that \[x_c = \frac{1}{1+t}x^* + \frac{t}{1+t}x_r.\] Note that $t\leq 2$ because $|x_c-l_k| = \frac{w_k}{2}$ and $|x_r - x_c|= \frac{w_k}{4}$, so \[t = \frac{|x^* - x_c|}{|x_r-x_c|} \leq \frac{|l_k-x_c|}{|x_r-x_c|} = \frac{w_k/2}{w_k/4} = 2.\]
Since $f$ is convex, \[f(x_c) \leq \frac{1}{1+t}f(x^*) + \frac{t}{1+t}f(x_r)\] and so
\begin{eqnarray*}
f(x^*) &\geq& (1+t)\left(f(x_c)-\frac{t}{1+t}f(x_r)\right) \\
&=& f(x_r) +(1+t)(f(x_c)-f(x_r)) \\
&\geq& f(x_r) - (1+t)|f(x_c)-f(x_r)| \\
&\geq& f(x_r) - (1+t)3H\gamma_i \\
&\geq& f(x_r) -9H\gamma_i.
\end{eqnarray*}
Thus for each $x \in \{x_l,x_c,x_r\}$, \[f(x) \leq f(x_r) + 3H\gamma_i \leq f(x^*) + 12H\gamma_i.\]
\end{proof}

\begin{lemma} [Lemma 4 in \cite{agarwal2011stochastic}] \label{lem:sb3}
Under $\event$, the total number of epochs $K$ is bounded by $\log_{4/3}(T)$. 
\end{lemma}
\begin{proof}
Observe that for any round $i$ that does not terminate the algorithm, $N_i = \frac{\log(T)}{\gamma_i^2} \leq T$ (since algorithm terminates upon reaching $T$ time steps), which implies $\gamma_i \geq \sqrt{\frac{\log(T)}{T}}$. Since $\gamma_{i+1} = \frac{1}{2}\gamma_{i}$, let us define $\gamma_{min} := \frac{1}{2}\sqrt{\frac{\log(T)}{T}}$ so that $\gamma_{min} \leq \gamma_i$ for any $\gamma_i$. Define the interval $I:= [x^*-\frac{H\gamma_{\min}}{ \max(h,p)}, x^*+\frac{H\gamma_{\min}}{ \max(h,p)}]$, so that for any $x\in I$, \[f(x)-f(x^*) \leq  \max(h,p)|x-x^*| \leq H\gamma_{\min}\] by Lemma \ref{l_factor}. Now, for any epoch $k'$ which ends in round $i'$, $H\gamma_{\min} \leq H\gamma_{i'}$ and hence by Lemma \ref{lem:sb1} we have 
\[I \subseteq \{x \in \range: f(x) \leq f(x^*) + H\gamma_{i'} \} \subseteq [l_{k'+1}, r_{k'+1}].\] So for any epoch $k'$, the length of interval $I$ is no more than the length of interval $[l_{k'+1}, r_{k'+1}]$, and so  
\[\frac{2H\gamma_{\min}}{ \max(h,p)} \leq w_{k'+1}.\] 
Since $w_{k'+1} \leq \frac{3}{4}w_{k'}$ for any $k'=1,2,\ldots,K-1$, we have that for $k' = K-1$,
\[\frac{2H\gamma_{\min}}{ \max(h,p)} =\frac{H}{ \max(h,p)}\sqrt{\frac{\log(T)}{T}} \leq w_{K} \leq (\frac{3}{4})^{K-1}w_1=(\frac{3}{4})(\frac{3}{4})^K(\rmax).\] 
Rearranging the inequality we get that 
\[K \leq \frac{1}{2}\log_{4/3}(\frac{9 \max(h,p)^2(\rmax)^2T}{16H^2\log(T)}) \leq \log_{4/3}(T)\]
since $H=\Hval$.
\end{proof}

\begin{lemma} [Lemma 3 in \cite{agarwal2011stochastic}] \label{lem:sb4}
Recall $T_{k,i,a}$ is the set of consecutive times where base stock policy level $x_a$ is played in round $i$ of epoch $k$, for $a \in \{l,c,r\}$. Then under $\event$, we can bound (over time steps of all such $k,i,a$)
\[\sum_{k, i, a, t\in T_{k,i,a}} f(x(t)) - f(x^*) \leq 146 H \log_{4/3}(T)\sqrt{T\log(T)}.\]
\end{lemma}
\begin{proof}
Let us first fix an epoch $k$ and assume it ends in round $i(k)$. If $i(k)=1$, then by Lemma \ref{l_factor},
\begin{equation} \label{eq:sbbound1}
\sum_{i, a, t\in T_{k,i,a}} (f(x_t) - f(x^*)) \leq 3N_1\max(h,p)|x_t-x^*| \leq \left(\frac{3\log(T)}{\gamma_1^2}\right)\max(h,p)\rmax.
\end{equation}
Otherwise, if $i(k) > 1$, then 
\[\sum_{i, a, t\in T_{k,i,a}} (f(x_t) - f(x^*)) = \sum_{i=1}^{i(k)-1} \sum_{a, t\in T_{k,i,a}} (f(x_t) - f(x^*)) + \sum_{a, t\in T_{k,i(k),a}} (f(x_t) - f(x^*)).\]
By Lemma \ref{lem:sb2}, for each $x_t \in \{x_r,x_c,x_l\}$, $f(x_t) - f(x^*) \leq 12 H\gamma_i$ for all $i= 1,2,\ldots,i(k)-1$. Also, $\gamma_{i(k)-1} = 2\gamma_{i(k)}$, so when $i(k) >1$,
\begin{eqnarray*}
\sum_{i, a, t\in T_{k,i,a}} (f(x_t) - f(x^*)) &\leq& \sum_{i=1}^{i(k)-1} \sum_{a, t\in T_{k,i,a}} (12H\gamma_i) + \sum_{a, t\in T_{k,i(k),a}} (12H\gamma_{i(k)-1})\\
&\leq& \sum_{i=1}^{i(k)-1} (3N_i) (12H\gamma_i) + (3N_{i(k)}) (24H\gamma_{i(k)}) \\
&\leq& \sum_{i=1}^{i(k)} 72N_iH\gamma_i \\
&\leq& \frac{72H\log(T)}{\gamma_{min}}, \\
\end{eqnarray*}
where in the last step we used a similar argument as (\ref{eq:geo_sum}). Combining this result with (\ref{eq:sbbound1}), we have for any number of rounds $i(k)$,
\begin{eqnarray*}
\sum_{i, a, t\in T_{k,i,a}} (f(x_t) - f(x^*)) &\leq& \left(\frac{3\log(T)}{\gamma_1^2}\right)\max(h,p)\rmax + \sum_{i=1} 72N_iH\gamma_i \\
\text{(since $\gamma_{min} \leq \gamma_1 = 1/2$)} &\leq& \left(\frac{6\log(T)}{\gamma_{min}}\right)\max(h,p)\rmax + \frac{72H\log(T)}{\gamma_{min}} \\
\text{(since $H = \Hval$)} &\leq& \frac{73H\log(T)}{\gamma_{min}}. \\
\end{eqnarray*}
Therefore, over all epochs $k$, by Lemma \ref{lem:sb3}
\begin{eqnarray*}
\sum_{k,i, a, t\in T_{k,i,a}} (f(x_t) - f(x^*)) & \le & \log_{4/3}(T)\left(\frac{73H\log(T)}{\gamma_{min}} \right), \\
\end{eqnarray*}
and the result follows from substituting $\gamma_{min} = \frac{1}{2}\sqrt{\frac{\log(T)}{T}}$.
\end{proof}

\end{document}